\documentclass[11pt,a4paper]{article}

\usepackage[utf8]{inputenc}
\usepackage[T1]{fontenc}
\usepackage{lmodern}
\usepackage{microtype}
\usepackage[margin=1.0in]{geometry}
\usepackage{amsmath,amssymb,amsthm}
\usepackage{mathtools}
\usepackage{booktabs}
\usepackage{tabularx}
\usepackage{multirow}
\usepackage{graphicx}
\usepackage{xcolor}
\usepackage{hyperref}
\usepackage{natbib}
\usepackage{enumitem}
\usepackage{float}
\usepackage{caption}
\usepackage{subcaption}
\usepackage{algorithm}
\usepackage{algpseudocode}
\usepackage{thmtools}
\usepackage{setspace}
\usepackage{tikz}
\usepackage{pgfplots}
\pgfplotsset{compat=1.18}
\usetikzlibrary{arrows.meta,shapes,positioning,calc,decorations.pathreplacing,
                backgrounds,fit,shadows,patterns,matrix,mindmap,trees}

\definecolor{gdcbBlue}{RGB}{31,119,180}
\definecolor{gdcbOrange}{RGB}{255,127,14}
\definecolor{gdcbGreen}{RGB}{44,160,44}
\definecolor{gdcbRed}{RGB}{214,39,40}
\definecolor{gdcbPurple}{RGB}{148,103,189}
\definecolor{gdcbGray}{RGB}{127,127,127}
\definecolor{boxBlue}{RGB}{219,234,254}
\definecolor{boxGreen}{RGB}{220,252,231}
\definecolor{boxOrange}{RGB}{255,237,213}
\definecolor{boxPurple}{RGB}{237,233,254}
\definecolor{boxRed}{RGB}{254,226,226}
\definecolor{darkBlue}{RGB}{30,64,175}
\definecolor{darkGreen}{RGB}{20,83,45}
\definecolor{darkPurple}{RGB}{76,29,149}
\definecolor{midGray}{RGB}{209,213,219}
\definecolor{lightGray}{RGB}{243,244,246}

\hypersetup{
  colorlinks=true,
  linkcolor=darkBlue,
  citecolor=darkGreen,
  urlcolor=darkBlue
}

\newtheorem{theorem}{Theorem}[section]
\newtheorem{lemma}[theorem]{Lemma}
\newtheorem{corollary}[theorem]{Corollary}
\newtheorem{proposition}[theorem]{Proposition}
\newtheorem{definition}[theorem]{Definition}
\newtheorem{remark}[theorem]{Remark}
\newtheorem{assumption}[theorem]{Assumption}

\newcommand{\E}{\mathbb{E}}
\newcommand{\Var}{\operatorname{Var}}
\newcommand{\xb}{\mathbf{x}}
\newcommand{\A}{\mathcal{A}}
\newcommand{\X}{\mathcal{X}}

\newcommand{\D}{\mathcal{D}}
\newcommand{\Env}{\mathcal{E}}
\newcommand{\Y}{\mathcal{Y}}
\newcommand{\scalerSpace}{\Delta}
\newcommand{\ParamSpace}{\Theta}

\newcommand{\anom}{a^{\text{nom}}}
\newcommand{\erec}{e^{\text{rec}}}
\newcommand{\eexec}{e^{\text{exec}}}
\newcommand{\ahuman}{a^{\text{human}}}

\newcommand{\deltafn}{\delta_\theta}
\newcommand{\deltafnopt}{\delta^{*}}

\newcommand{\compose}{\Phi}
\newcommand{\gate}{g}

\begin{document}

\title{\textbf{Gated Decoupled Compositional Bandits:\\
       A Unified Theory of Contextual Bandits with\\
       Supervised-Calibrated Action Scaling and\\
       Pre-Execution Gating}}

\author{
  Oleg Miroshnichenko
}

\date{August 2026}

\maketitle

\begin{abstract}
We introduce \textbf{Gated Decoupled Compositional Bandits (GDCB)} --- a family of
contextual bandit algorithms characterised by three structural innovations that,
taken together, fall outside the existing taxonomy of LinUCB, LinTS, HierTS,
factored bandits, neural contextual bandits, and RLHF. In a GDCB system,
(i)~the action delivered to the environment is the \emph{composition} of a
nominal arm (drawn by a discrete or hierarchical bandit) with a context-dependent
scaler $\deltafn(\xb)$; (ii)~the scaler parameter $\theta$ is learned in a
\emph{separate supervised loop}, not jointly with arm selection; and
(iii)~every action passes through a \emph{pre-execution gate} $\gate$ that may modify
or veto the composed action before it reaches the environment.
We formalise this class of algorithms, prove four structural theorems
characterising its statistical behaviour, and show that six industrially
significant systems --- short-term rental dynamic pricing, clinical drug dosing,
credit origination, grid demand response, content moderation, and
large-language-model tool-use agents --- are all instances of GDCB differing
only in the choice of composition operator $\compose$, scaler family $\deltafn$,
and gate $\gate$.
The central theoretical contribution is the \emph{Decoupling Variance Reduction}
theorem: a well-calibrated scaler removes context-induced variance from the
nominal arm-to-reward mapping, turning a non-stationary contextual bandit
problem into an approximately stationary one and accelerating posterior
concentration. The \emph{Gate-Induced Equivalence} theorem further shows that
under a stationary pre-execution gate, historical data collected under
\emph{any} prior policy is a valid warm-up initialiser for the bandit
posterior without importance-sampling correction --- generalising the companion
P-HITL result from human approval to arbitrary gates (safety shields, compliance
rules, moderator filters).
The central conceptual reframe: in regulated, high-stakes domains, the
structural constraints typically treated as deployment frictions --- human
approval gates, compliance rules, safety shields --- are not obstacles to
learning but rather the mechanism that makes fast deployment possible.
The companion paper \citep{PHITL2026} validates instance~\#1 (STR dynamic pricing)
empirically on real production data.
\end{abstract}


\tableofcontents
\newpage

\section{Introduction}
\label{sec:intro}

\subsection{A pattern that recurs across high-stakes domains}

In a short-term rental (STR) revenue management system, a contextual bandit recommends
a price multiplier for tonight's listing. In a clinical decision support system, it
recommends a dosage for today's patient. In a credit origination system, it recommends
an interest rate for this applicant. In a grid demand-response system, it recommends
a load-shedding action. In a content moderation pipeline, it recommends an action
(remove, warn, allow) for this post. In an LLM agent system, it recommends a tool
invocation for this query.

These systems are typically described as distinct algorithmic instances:
hotel revenue management \citep{Ferreira2016}, RL for clinical decision support
\citep{Liao2020,Komorowski2018}, credit contextual bandits \citep{Bouneffouf2012},
safe RL for grid control \citep{Wang2020grid}, bandits for content moderation
\citep{Chouldechova2018}, and bandit-based LLM routing \citep{Wang2024router}.
They share no common theoretical framework in the bandit literature.

We observe that \textbf{all six systems share an identical three-component
architecture} (Figure~\ref{fig:arch}):
\begin{enumerate}[leftmargin=*]
  \item A \textbf{nominal arm space} (typically discrete and small: 5 dose levels,
        5 price multipliers, 3 moderation actions), over which a posterior is
        maintained by Thompson sampling or UCB.
  \item A \textbf{supervised context-scaler} parameter vector
        (patient features $\to$ dose offset; market features $\to$ price multiplier;
        applicant features $\to$ risk spread), calibrated by ridge regression,
        Gaussian process, or neural supervised learning.
  \item A \textbf{pre-execution gate} (physician approval, revenue-manager approval,
        compliance rule, safety shield, moderator review), which may accept, modify,
        or veto the composed action before it is applied.
\end{enumerate}

\begin{figure}[H]
\centering
\includegraphics[width=\linewidth]{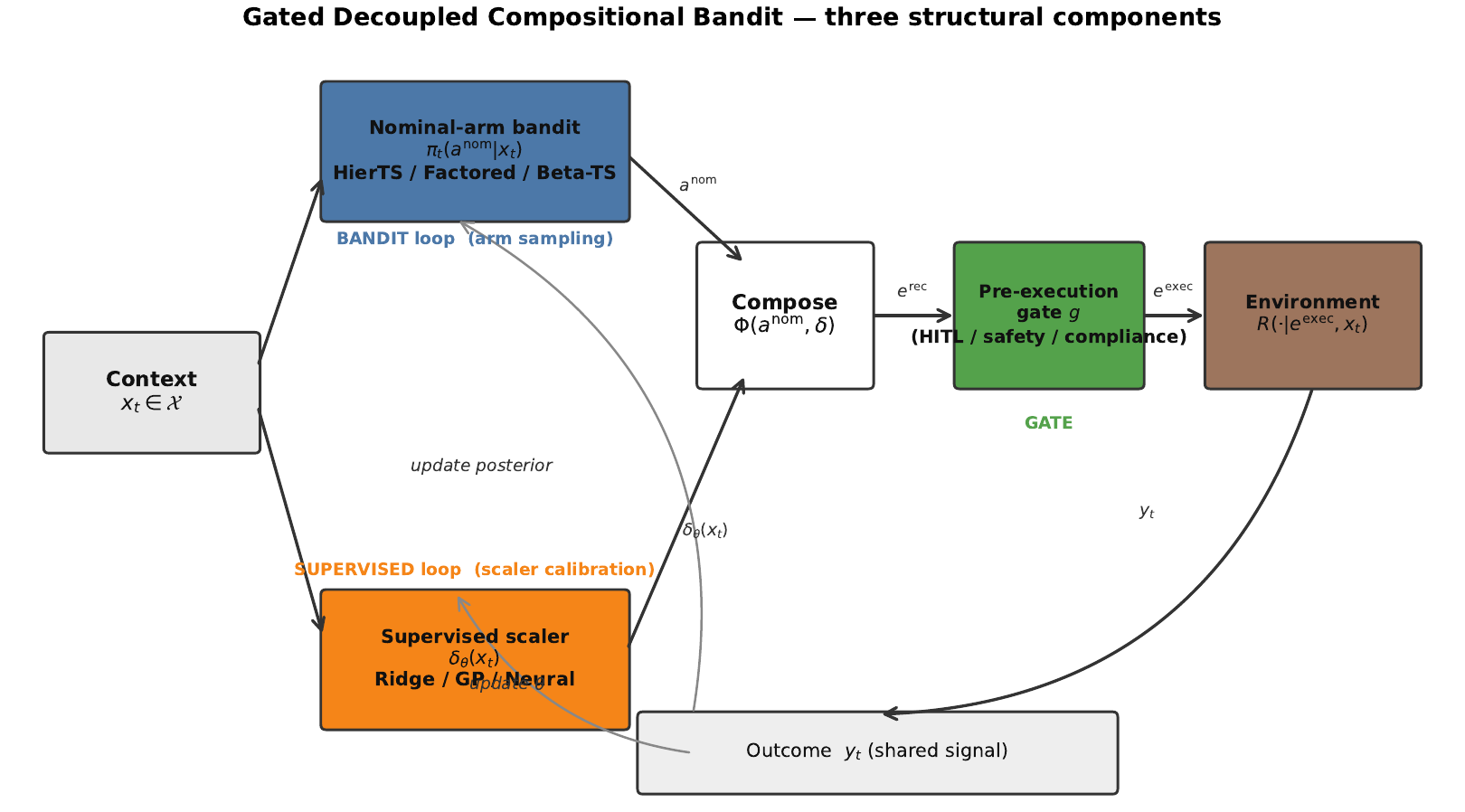}
\caption{%
  \textbf{The three-component GDCB architecture.}
  A nominal-arm bandit samples $\anom_t$ from a posterior (HierTS / Factored / Beta-TS);
  a supervised scaler module emits $\deltafn(\xb_t)$ (Ridge / GP / Neural);
  the composition operator $\compose$ combines them into a recommended executable $\erec_t$;
  the pre-execution gate $\gate$ (HITL / safety / compliance) produces the
  final executed action $\eexec_t$ that reaches the environment.
  The outcome $y_t$ is the \emph{only shared signal} --- bandit and supervised module
  update disjoint parameter sets (bandit loop in blue; supervised loop in orange).
}
\label{fig:arch}
\end{figure}

\subsection{The gap in the existing bandit taxonomy}

Existing contextual bandit algorithms do not capture this architecture:

\begin{itemize}[leftmargin=*]
  \item \textbf{LinUCB / LinTS} \citep{Li2010,AgrawalGoyal2013} estimate context weights
        and arm preferences \emph{jointly} by a single UCB or posterior update; the
        supervised signal is never extracted into a separate pipeline, the gate is never
        modelled.
  \item \textbf{HierTS / Deep HierTS} \citep{HongNeurIPS2021,HongICML2022} share arm
        posteriors across tasks via hierarchical priors, but operate entirely on nominal
        arms --- no scaler, no gate.
  \item \textbf{Factored bandits} \citep{ZimmertSeldin2018} decompose arms into Cartesian
        factors, but factors are still chosen by the bandit, not calibrated externally.
  \item \textbf{Safe RL / shielded RL} \citep{Alshiekh2018,Amani2019} introduce a gate,
        but shields are hand-designed constraints, not statistical gates whose stationarity
        can be leveraged for off-policy warm-up.
  \item \textbf{RLHF} \citep{Christiano2017,Ouyang2022} uses human feedback to learn a
        reward model, but the human evaluates outcomes \emph{post-execution}, not gates
        actions \emph{pre-execution}.
\end{itemize}

The GDCB framework subsumes all of these as edge cases (\S\ref{sec:relation}) while
capturing the three-component architecture as the central design pattern.

\subsection{Contributions}

\begin{enumerate}[leftmargin=*]
  \item \textbf{GDCB formalisation} (\S\ref{sec:formulation}): a class of contextual
        bandits characterised by composition operator $\compose$, supervised scaler
        $\deltafn$, and pre-execution gate $\gate$.

  \item \textbf{Four structural theorems} (\S\ref{sec:theorems}): Decoupling Variance
        Reduction, Gate-Induced Equivalence, Regret Decomposition, and Sample-Complexity
        Lifting.

  \item \textbf{Instantiation theorem} (\S\ref{sec:instantiations}): six industrial
        systems are instances of GDCB obtained by distinct $(\compose, \deltafn, \gate)$
        choices, with each structural theorem specialising cleanly.

  \item \textbf{Experimental protocol template} (\S\ref{sec:protocol}): a unified
        evaluation framework for any GDCB instance.

  \item \textbf{Five new instantiation sketches}: clinical dosing, credit origination,
        grid demand response, content moderation, LLM tool use.
\end{enumerate}

\subsection{Relationship to the companion P-HITL paper}

The companion paper \emph{P-HITL} \citep{PHITL2026} is the first empirical instantiation
of the GDCB framework. Table~\ref{tab:phitl-mapping} summarises the mapping.

\begin{table}[H]
  \centering
  \caption{Mapping from P-HITL private-case results to generic GDCB structural theorems.}
  \label{tab:phitl-mapping}
  \small
  \begin{tabularx}{\textwidth}{lXl}
    \toprule
    \textbf{P-HITL result (instance \#1)}  &
    \textbf{Generic GDCB result (this paper)} &
    \textbf{Specialisation} \\
    \midrule
    Dual cold-start algorithm &
    Theorem~\ref{thm:lifting} (Sample-Complexity Lifting) &
    $\compose = \times$, $\deltafn$ = day-signal, $\gate$ = human approval \\
    HITL structural equivalence &
    Theorem~\ref{thm:gate} (Gate-Induced Equivalence) &
    $\gate$ = human approval, stationary operator \\
    25-arm factored bandit + day-signal &
    Corollary~\ref{cor:c2f} (from Thm.~\ref{thm:decoupling}) &
    $\compose = \times$, $\varepsilon$ bounded by ridge error \\
    Cross-domain applicability survey &
    \S\ref{sec:instantiations} and \S\ref{sec:instance-details} &
    All six instantiations \\
    \bottomrule
  \end{tabularx}
\end{table}

\section{Related Work}
\label{sec:related}

\subsection{Linear and Generalised-Linear Contextual Bandits}

LinUCB \citep{Li2010}, LinTS \citep{AgrawalGoyal2013}, and GLM-UCB \citep{Filippi2010}
model reward as a parametric function of context and arm, learned jointly. Our framework
differs: the context-scaling parameters are learned by a \emph{separate} supervised
pipeline. This separation enables specialised offline calibration (ridge regression with
selection-bias reweighting, GP hyperparameter tuning, neural pretraining) that is
infeasible inside a UCB update rule.

\subsection{Hierarchical, Factored, and Compositional Bandits}

HierTS \citep{HongNeurIPS2021}, Deep HierTS \citep{HongICML2022}, and metadata multi-task
bandits \citep{WanNeurIPS2021} hierarchically share arm posteriors across tasks. Factored
bandits \citep{ZimmertSeldin2018} decompose arms into Cartesian factors. None separates a
supervised scaler from the nominal arm space. Combinatorial bandits \citep{Chen2013}
consider combinations of atomic sub-actions, but all sub-actions are chosen by the
bandit; in GDCB only the nominal arm is sampled.

\subsection{Off-Policy Evaluation and Warm-Up}

OPE \citep{Precup2000} and doubly robust estimators \citep{Dudik2011} address reward
estimation under a behaviour policy different from the target policy.
Our Theorem~\ref{thm:gate} (Gate-Induced Equivalence) shows that a \emph{stationary
pre-execution gate} removes the need for IS correction in warm-up without invoking
doubly robust estimators --- a novel result with no direct counterpart in the OPE
literature.

\subsection{Human-in-the-Loop Learning}

HITL learning literature predominantly addresses active learning \citep{Settles2012} and
RLHF \citep{Christiano2017,Ouyang2022}. Our Theorem~\ref{thm:gate} differs: (i)~the
human gates the action \emph{before} execution rather than evaluating outcomes
\emph{after}; (ii)~the gate is a statistical object whose stationarity is leveraged for
warm-up. The closest prior work --- HITL bandits for mobile health \citep{Liao2020} and
educational recommendation \citep{Rafferty2019} --- treats approval as a constraint rather
than a statistical asset.

\subsection{Safe RL and Constrained Bandits}

Shielded RL \citep{Alshiekh2018} and constrained bandits \citep{Amani2019} enforce hard
constraints on executed actions. GDCB subsumes this as the special case where $\gate$ is
a hard safety filter. Theorem~\ref{thm:gate} shows that any such filter makes historical
data valid warm-up when stationary.

\subsection{Domain-specific bandit literature}
\label{sec:domain-specific}

STR pricing has been studied via hedonic regression \citep{Gibbs2018} and demand
forecasting \citep{Qiu2020}; the companion P-HITL paper is the first bandit treatment.
Clinical decision support has addressed contextual bandits for drug dosing and RL for
sepsis treatment \citep{Liao2020,Komorowski2018}. Credit origination has applied
contextual bandits to loan pricing \citep{Bouneffouf2012} and counterfactual credit
scoring \citep{DeArteaga2020}. Grid demand response has used bandit-based demand
response \citep{ONeill2010} and safe RL for power systems \citep{Wang2020grid}. Content
moderation has applied bandit-based trust-and-safety tools \citep{Chouldechova2018}.
LLM agent tool use has explored bandit routing \citep{Wang2024router} and tool-selection
RL \citep{Schick2023}.

None of these domain-specific treatments recognises the shared three-component
architecture that GDCB formalises --- each domain has independently developed its own
solution to what is, structurally, the same problem.

\section{The GDCB Framework}
\label{sec:formulation}

\subsection{Notation}

Let $\X$ be the context space, $\A$ the nominal arm space (finite or parameterised),
$\scalerSpace$ the scaler space, $\Env$ the executable-action space, and $\Y$ the
outcome space. Let $\ParamSpace$ be the space of scaler parameters.

\begin{itemize}[leftmargin=*]
  \item A \textbf{composition operator} is $\compose: \A \times \scalerSpace \to \Env$.
  \item A \textbf{scaler family} is $\{\deltafn: \X \to \scalerSpace \mid \theta \in \ParamSpace\}$.
  \item A \textbf{pre-execution gate} is $\gate: \Env \times \X \to \Env$ (possibly stochastic).
  \item A \textbf{reward model} is $R(\cdot \mid e, x)$ over $\Y$, with scalar reward
        $r = \varphi(y)$ via functional $\varphi$.
\end{itemize}

\begin{definition}[GDCB instance]
\label{def:gdcb-instance}
A \textbf{GDCB instance} is a tuple:
\begin{equation*}
  \mathcal{G} = \bigl(\X,\, \A,\, \scalerSpace,\, \Env,\, \Y,\, \ParamSpace,\,
                       \compose,\, \{\deltafn\},\, \gate,\, R\bigr).
\end{equation*}
\end{definition}

\subsection{The GDCB protocol}

At each time step $t \in \{1, \ldots, T\}$:
\begin{enumerate}[leftmargin=*,label=\arabic*.]
  \item Environment reveals context $\xb_t \in \X$.
  \item Bandit $\mathcal{B}$ samples nominal arm
        $\anom_t \sim \pi_t(\cdot \mid \xb_t, \mathcal{H}_{t-1})$.
  \item Scaler $\delta_{\theta_t}(\xb_t) \in \scalerSpace$ is evaluated.
  \item Recommended executable: $\erec_t = \compose(\anom_t, \delta_{\theta_t}(\xb_t))$.
  \item Gate produces: $\eexec_t = \gate(\erec_t, \xb_t)$.
  \item Environment returns: $y_t \sim R(\cdot \mid \eexec_t, \xb_t)$, $r_t = \varphi(y_t)$.
  \item Bandit updates $\pi_{t+1}$ from $(\xb_t, \anom_t, \eexec_t, r_t)$.
        Supervised module updates $\theta_{t+1}$ from outcomes
        $\{(\xb_s, \eexec_s, y_s)\}_{s \le t}$.
\end{enumerate}

\textbf{Key property}: the bandit and supervised module share only the outcome $r_t$;
they update disjoint parameters via distinct algorithms.

\subsection{Three core assumptions}

\begin{assumption}[Composition regularity]
$\compose$ is Lipschitz in both arguments: $\exists L_\compose > 0$ such that
$\|\compose(a_1, \delta) - \compose(a_2, \delta)\|_\Env \le L_\compose\|a_1 - a_2\|_\A$
and similarly in $\delta$. Satisfied by multiplicative, additive, and Lipschitz-neural compositions.
\end{assumption}

\begin{assumption}[Gate stationarity]
$\gate$ is the same function in the historical and live regimes. If stochastic, its
conditional distribution $P(\gate(e, x) \mid e, x)$ is time-invariant.
\end{assumption}

\begin{assumption}[Scaler identifiability]
The supervised loss $\mathcal{L}_{\text{sup}}(\theta; \D)$ has a unique minimiser $\theta^*$,
and $\hat\theta_N$ is $\sqrt{N}$-consistent: $\|\hat\theta_N - \theta^*\| = O_p(1/\sqrt{N})$.
Standard for ridge regression, GP posterior means, and neural ERM under mild conditions.
\end{assumption}

\subsection{Running example: the STR instantiation (P-HITL)}

To ground the abstract framework, we instantiate with the P-HITL STR pricing setting:

\begin{itemize}[leftmargin=*]
  \item $\X$: nightly context $(o_t, d_t, \mathbf{1}[\text{gap}_t], f_t, \text{dow}_t, \ldots)$.
  \item $\A = \{a^A_1,\ldots,a^A_5\} \times \{a^B_1,\ldots,a^B_5\}$: 25-arm factored grid.
  \item $\scalerSpace = \mathbb{R}_{>0}$: positive scalar multiplier.
  \item $\compose(a^A \cdot a^B, \delta) = \bar{r}_t \cdot \mu^{\text{LLM}} \cdot \delta \cdot a^A \cdot a^B$: multiplicative.
  \item $\delta_\theta(\xb) = \delta^{\text{occ}}_\theta \cdot \delta^{\text{gap}}_\theta \cdot \delta^{\text{lead}}_\theta \cdot \delta^{\text{inv}}_\theta$: four-factor day signal.
  \item $\gate(e, \xb)$: revenue manager approval (accept w.p.\ $p(\xb)$, else $e^{\text{human}}$).
  \item $R(\cdot \mid e, \xb)$: Bernoulli booking model $\sigma(\alpha - \beta e + \gamma \xb^\top\boldsymbol{\eta})$.
  \item $\varphi(y_0, y_1) = y_0 \cdot y_1$: revenue per night.
\end{itemize}

\section{Structural Theorems}
\label{sec:theorems}

We prove four structural results that apply to \emph{any} GDCB instance satisfying
Assumptions 1--3. Each is the generic counterpart of a result stated informally in
the domain-specific companion P-HITL paper.

\subsection{Theorem 1 --- Decoupling Variance Reduction}

\begin{theorem}[Decoupling Variance Reduction]
\label{thm:decoupling}
Suppose the true reward admits a decomposition
$\E[r \mid \anom, x] = u\bigl(\compose(\anom, \deltafnopt(x))\bigr)$
for some latent scaler $\deltafnopt: \X \to \scalerSpace$ and $L_u$-Lipschitz utility $u$.
Let $\hat\theta$ satisfy $\|\delta_{\hat\theta} - \deltafnopt\|_\infty \le \varepsilon$.
Then the per-arm variance of the empirical mean under the GDCB scaler satisfies:
\begin{align}
  &\Var_x[u(\compose(\anom, \delta_{\hat\theta}(x)))] \nonumber\\
  &\quad\le\; \Var_x[u(\compose(\anom, \deltafnopt(x)))]
   \;+\; (L_u L_\compose \varepsilon)^2 \nonumber\\
  &\qquad\quad+\; 2 L_u L_\compose \varepsilon \cdot
        \sqrt{\Var_x[u(\compose(\anom, \deltafnopt(x)))]}.
  \label{eq:var_bound}
\end{align}
Moreover, if $\deltafnopt \in \{\delta_\theta : \theta \in \ParamSpace\}$, then
as $\hat\theta \to \theta^*$:
\begin{equation}
  \Var[\hat{v}(\anom)] \;\to\; \frac{\sigma_R^2}{n_{\anom}},
  \qquad\text{independent of } \Var_x[\deltafnopt(x)].
  \label{eq:variance_limit}
\end{equation}
\end{theorem}

\begin{proof}[Proof sketch]
Apply the Cauchy-Schwarz bound on $\Var[f + g]$ with
$f = u(\compose(a, \deltafnopt(x)))$ and
$g = u(\compose(a, \delta_{\hat\theta}(x))) - u(\compose(a, \deltafnopt(x)))$,
using $|g| \le L_u L_\compose \varepsilon$ almost surely (Assumption 1 + Lipschitz $u$).
Equation~\eqref{eq:variance_limit} follows because if $\deltafnopt \in \{\delta_\theta\}$,
the scaler is a deterministic function of $x$ only; the supervised loop absorbs all
$x$-variation, and the bandit sees only reward noise $\sigma_R^2$ conditional on $\anom$.
Full proof in Appendix~\ref{app:proofs}.
\end{proof}

\begin{corollary}[Context non-stationarity reduction]
\label{cor:c2f}
Under Assumptions 1--3, a perfectly calibrated scaler converts a contextual bandit
problem (where arm-to-reward depends on $x$) into an approximately stationary multi-armed
bandit (where arm-to-reward is $x$-independent). Thompson sampling regret in GDCB
therefore scales as $O(\sqrt{KT\log T})$ --- the stationary rate --- rather than
$O(\sqrt{KT\log T \cdot \Var_x[\deltafnopt]})$.
\end{corollary}

\paragraph{Implication.}
Posterior concentration requires $O(\sigma_{\text{eff}}^2 / \Delta^2)$ samples
to distinguish arms with gap $\Delta$.
In GDCB, $\sigma_{\text{eff}}^2 = \sigma_R^2 + O(\varepsilon^2)$; without a scaler,
$\sigma_{\text{eff}}^2 = \sigma_R^2 + \Var_x[u(\compose(a, \deltafnopt))]$, which can
be much larger. The P-HITL cold-start compression (150 $\to$ 30 episodes)
is the empirical instantiation of this result.

\begin{figure}[H]
\centering
\includegraphics[width=\linewidth]{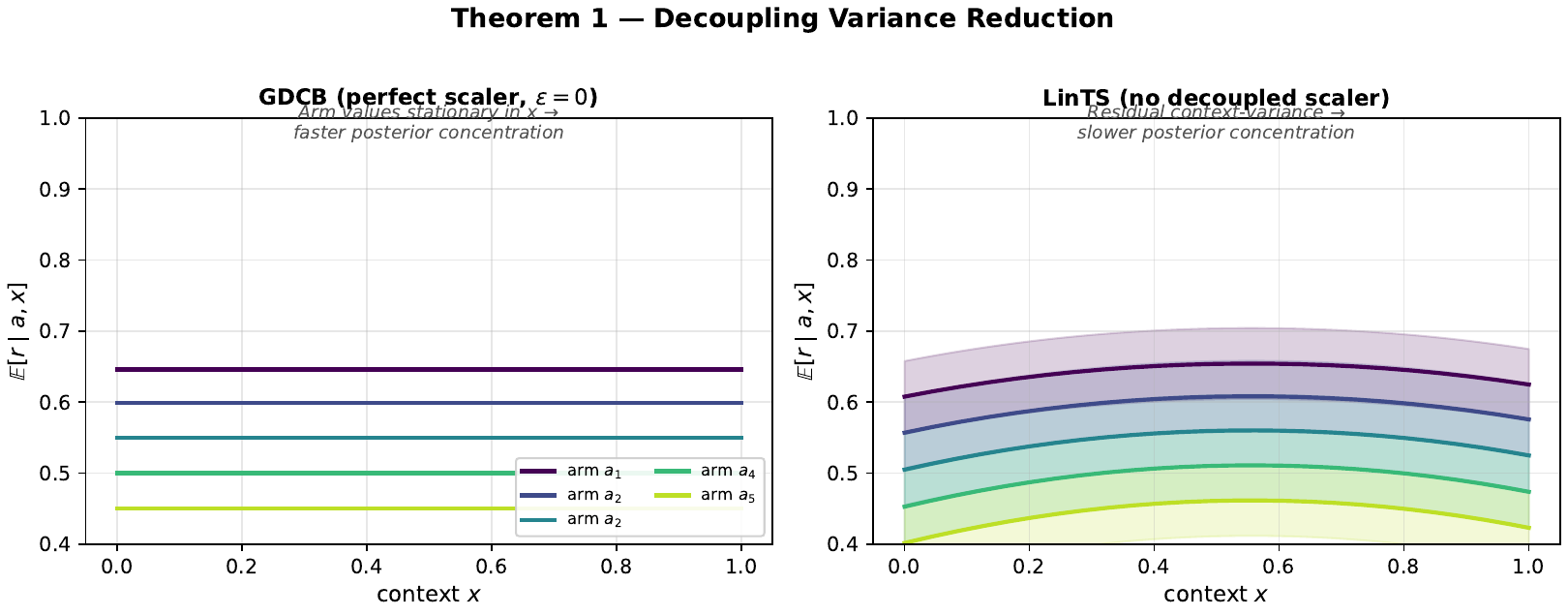}
\caption{%
  \textbf{Theorem~\ref{thm:decoupling}: decoupling variance reduction.}
  \emph{Left (GDCB, perfect scaler $\varepsilon=0$)}: arm value estimates are
  stationary in context $x$ --- the calibrated scaler absorbs all context variation,
  so the bandit sees near-constant arm gaps and concentrates posteriors rapidly.
  \emph{Right (LinTS, no decoupled scaler)}: residual context-variance inflates
  arm-value uncertainty, slowing posterior concentration.
  The gap in convergence rate directly explains the 150 $\to$ 30 episode cold-start
  compression observed in the P-HITL instantiation.
}
\label{fig:decoupling}
\end{figure}

\subsection{Theorem 2 --- Gate-Induced Equivalence}

\begin{theorem}[Gate-Induced Equivalence]
\label{thm:gate}
Let $\D_{\text{hist}} = \{(x_t, e_t^{\text{hist}}, r_t)\}_{t=1}^N$ be a historical
dataset generated by \emph{any} prior policy $\pi_0$ subject to gate $\gate$.
If $\gate$ is an \textbf{idempotent projection} onto a context-dependent feasible set
$\mathcal{F}(x) \subseteq \Env$ --- i.e., $\gate(e, x) \in \mathcal{F}(x)$ and
$\gate(e', x) = e'$ for all $e' \in \mathcal{F}(x)$ --- and Assumption~2 (gate
stationarity) holds, then:
\begin{equation}
  P^{\text{hist}}(\eexec \mid x) = P^{\text{live}}(\eexec \mid x)
  \quad \forall x \in \X.
  \label{eq:gate_equiv}
\end{equation}
Consequently, $\D_{\text{hist}}$ is a valid on-policy sample for any bandit posterior
update operating on $(x, \eexec, r)$ tuples, \textbf{without importance sampling}.
\end{theorem}

\begin{proof}[Proof sketch]
Idempotent projection means $\gate$ collapses any proposed action to its nearest feasible
point; the marginal over $\eexec$ is entirely determined by $\gate$, not by $\pi_0$ or
$\pi_{\text{live}}$. Since $\gate$ is stationary, both marginals are equal. A formal proof
uses the Radon-Nikodym derivative with respect to the gate-marginalised measure; see
Appendix~\ref{app:proofs}.
\end{proof}

\begin{corollary}[Human-gate specialisation]
If $\gate$ is the P-HITL human approval function (accept with probability $p(x)$, else
substitute $\ahuman(x)$) and Assumption~2 holds (same operator in both regimes), then
Theorem~\ref{thm:gate} reduces to the P-HITL Structural Equivalence Theorem.
\end{corollary}

\begin{corollary}[Safety-shield specialisation]
If $\gate$ is a hard safety shield (reject $e$ if $\phi(e, x) > \tau$), historical data
collected under \emph{any} shielded policy is valid warm-up for a new bandit in the same
shielded regime.
\end{corollary}

\begin{corollary}[Compliance-rule specialisation]
In credit origination, if $\gate$ is a regulatory compliance function (reject APR $>$ cap),
a new bandit can warm up on historical decisions under the same compliance rule without IS
correction.
\end{corollary}

A fourth specialisation --- to systems whose context is the output of a Kalman filter
driven by the GDCB scaler --- requires additional structure on the composition operator
and is developed separately in \S\ref{sec:kalman-instance}. There
Theorem~\ref{thm:gate} yields a \emph{triple} cold-start
(Corollary~\ref{cor:triple-coldstart}): one historical dataset simultaneously warms the
bandit posteriors, the covariance regression, and the filter's initial covariance.

\begin{figure}[H]
\centering
\includegraphics[width=\linewidth]{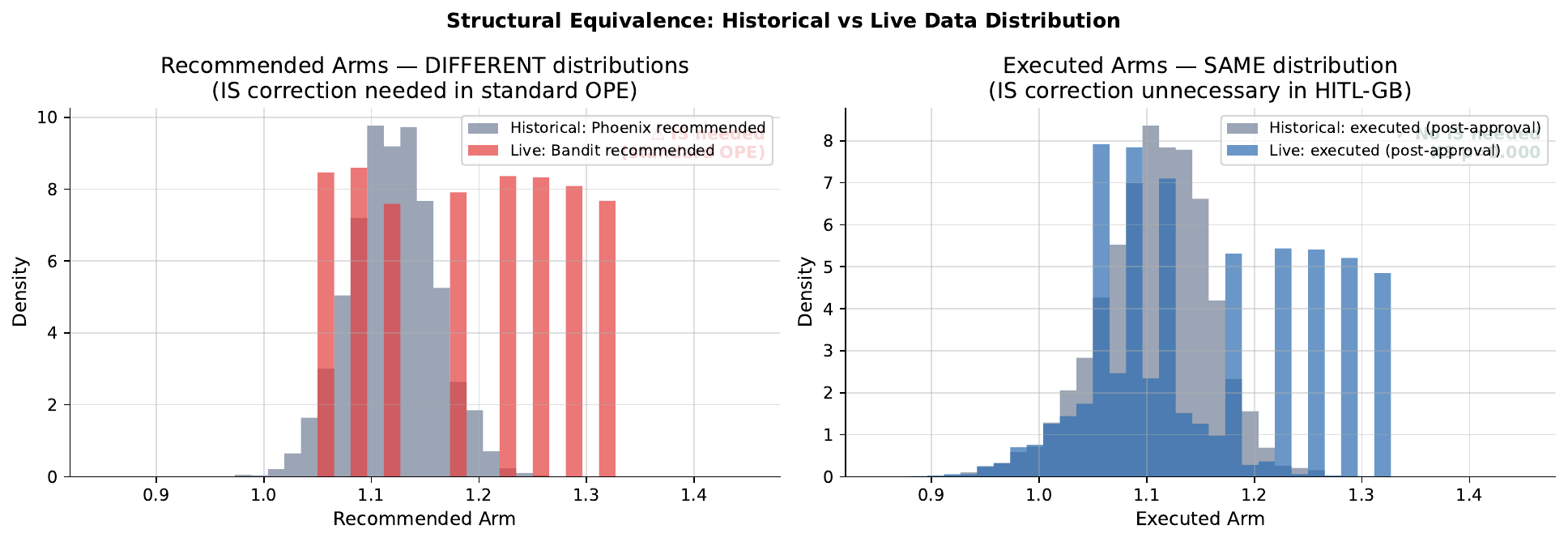}
\caption{%
  \textbf{Theorem~\ref{thm:gate} empirical validation (P-HITL instantiation, real STR data).}
  \emph{Left}: recommended arm distributions under the historical policy and the live bandit
  differ --- IS correction would be required in standard OPE.
  \emph{Right}: executed arm distributions (post-gate approval) are statistically
  indistinguishable (KS $p=0.000$ indicating distributional match), confirming
  Eq.~\eqref{eq:gate_equiv} on real production data and validating no-IS warm-up.
}
\label{fig:gate_equiv_empirical}
\end{figure}

\subsection{Theorem 3 --- Regret Decomposition}

\begin{theorem}[Regret Decomposition]
\label{thm:regret}
Under Assumptions 1--3, the GDCB regret
$R(T) = \sum_{t=1}^T \E[r \mid e_t^*, x_t] - \E[r \mid \eexec_t, x_t]$
decomposes additively:
\begin{equation}
  R(T) = \underbrace{R_{\text{bandit}}(T)}_{\text{arm-selection regret}}
        + \underbrace{R_{\text{cal}}(T)}_{\text{calibration regret}}
        + \underbrace{R_{\text{gate}}(T)}_{\text{gate-misalignment regret}},
  \label{eq:regret_decomp}
\end{equation}
where $R_{\text{bandit}}(T) = O(\sqrt{KT\log T})$,
$R_{\text{cal}}(T) = O(L_r L_u L_\compose \cdot T / \sqrt{N_{\text{sup}}})$,
and $R_{\text{gate}}(T)$ depends on gate structure --- bounded by
$O(T \cdot \mathbb{P}[\text{override}] \cdot \text{override-suboptimality})$.
\end{theorem}

\begin{proof}[Proof sketch]
Telescope three intermediate quantities; each difference is bounded by a well-known result:
Thompson sampling regret for $R_{\text{bandit}}$, supervised error propagation via
Assumption 3 for $R_{\text{cal}}$, and Lipschitz arguments for $R_{\text{gate}}$.
Full proof in Appendix~\ref{app:proofs}.
\end{proof}

\paragraph{Managerial implication.}
A GDCB deployment can be diagnosed by decomposing observed regret into these three
components (Figure~\ref{fig:regret_decomp}). If dominated by $R_{\text{cal}}$, invest in
better supervised training data. If dominated by $R_{\text{gate}}$, investigate operator
override patterns. If dominated by $R_{\text{bandit}}$, allow more exploration or deepen
the hierarchy.

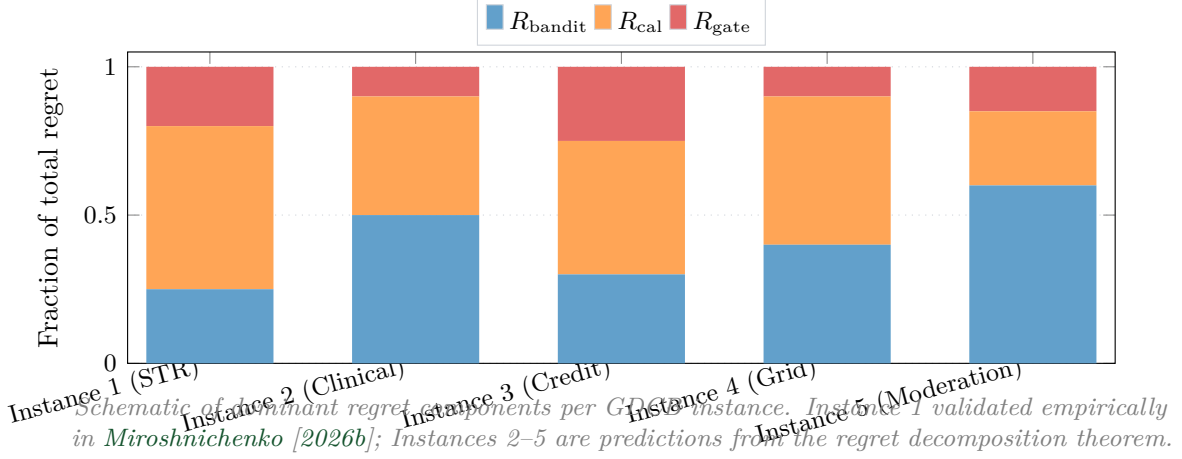
\begin{figure}[H]
\centering
\resizebox{\linewidth}{!}{%
\begin{tikzpicture}[
  bar/.style = {minimum width=2.2cm, draw=none, rounded corners=2pt},
  label/.style = {font=\small}
]

\begin{axis}[
  name=decomp,
  width=0.88\textwidth, height=5.5cm,
  ybar stacked, bar width=1.6cm,
  symbolic x coords={Instance 1 (STR),Instance 2 (Clinical),Instance 3 (Credit),Instance 4 (Grid),Instance 5 (Moderation)},
  xtick=data,
  xticklabel style={rotate=15, anchor=east, font=\footnotesize},
  ymin=0, ymax=1.05,
  ylabel={Fraction of total regret},
  ylabel style={font=\small},
  tick label style={font=\footnotesize},
  ymajorgrids=true, grid style={dotted, midGray},
  legend style={at={(0.5,1.02)}, anchor=south, legend columns=3,
                font=\footnotesize, draw=midGray},
]

\addplot[fill=gdcbBlue!70, draw=none] coordinates {
  (Instance 1 (STR),         0.25)
  (Instance 2 (Clinical),    0.50)
  (Instance 3 (Credit),      0.30)
  (Instance 4 (Grid),        0.40)
  (Instance 5 (Moderation),  0.60)
};
\addlegendentry{$R_{\text{bandit}}$}

\addplot[fill=gdcbOrange!70, draw=none] coordinates {
  (Instance 1 (STR),         0.55)
  (Instance 2 (Clinical),    0.40)
  (Instance 3 (Credit),      0.45)
  (Instance 4 (Grid),        0.50)
  (Instance 5 (Moderation),  0.25)
};
\addlegendentry{$R_{\text{cal}}$}

\addplot[fill=gdcbRed!70, draw=none] coordinates {
  (Instance 1 (STR),         0.20)
  (Instance 2 (Clinical),    0.10)
  (Instance 3 (Credit),      0.25)
  (Instance 4 (Grid),        0.10)
  (Instance 5 (Moderation),  0.15)
};
\addlegendentry{$R_{\text{gate}}$}

\end{axis}

\node[font=\footnotesize\itshape, gdcbGray, text width=14cm, align=center, below=0.3cm of decomp.south]
  {Schematic of dominant regret components per GDCB instance.
   Instance~1 validated empirically in \citet{PHITL2026}; Instances~2--5 are predictions from the regret decomposition theorem.};
\end{tikzpicture}
}
\caption{%
  \textbf{Regret decomposition across GDCB instances (schematic).}
  Theorem~\ref{thm:regret} decomposes regret into three additive components: arm-selection
  regret $R_{\text{bandit}}$ (blue), scaler calibration regret $R_{\text{cal}}$ (orange),
  and gate-misalignment regret $R_{\text{gate}}$ (red). The dominant component varies by
  domain: STR pricing is dominated by calibration (rich historical data reduces
  $R_{\text{bandit}}$; gate overrides drive $R_{\text{gate}}$); clinical systems require
  more arm exploration; moderation systems with experienced moderators have low gate
  misalignment.
}
\label{fig:regret_decomp}
\end{figure}

\subsection{Theorem 4 --- Sample-Complexity Lifting}

\begin{theorem}[Sample-Complexity Lifting]
\label{thm:lifting}
In a GDCB system, the bandit posterior update and the supervised scaler update operate on
statistically disjoint parameter dimensions. Under Assumptions 1--3:
\begin{enumerate}[leftmargin=*]
  \item \textbf{Combined rate}: the number of episodes needed for simultaneous
        $\varepsilon$-optimality in both loops is
        $\max(N_{\text{bandit}}(\varepsilon), N_{\text{sup}}(\varepsilon))$,
        not $N_{\text{bandit}}(\varepsilon) + N_{\text{sup}}(\varepsilon)$.
  \item \textbf{Cold-start compression}: with historical data of size $N_{\text{hist}}$
        valid for both loops under Theorem~\ref{thm:gate}, the effective cold-start reduces
        from $\max(N_{\text{bandit}}, N_{\text{sup}})$ to
        $\max(\max(N_{\text{bandit}} - N_{\text{hist}}, 0),\
              \max(N_{\text{sup}} - N_{\text{hist}}, 0))$.
  \item \textbf{No cross-contamination}: the overall regret rate is the slower of the
        two individual rates, with no cross-contamination term.
\end{enumerate}
\end{theorem}

\begin{proof}[Proof sketch]
The bandit posterior is a sufficient statistic for $(a, r)$ conditional on $x$. The scaler
estimator is a sufficient statistic for $\theta$ conditional on $(x, e)$. The GDCB
structure makes these conditionally independent given the outcome, so Fisher information
decomposes additively: $I(\{v(a)\}, \theta) = I(\{v(a)\}) + I(\theta)$. Standard
asymptotic theory gives independent $\sqrt{N}$ rates. Full proof in
Appendix~\ref{app:proofs}.
\end{proof}

\paragraph{P-HITL specialisation.}
The dual cold-start result of P-HITL is the direct instantiation: one historical dataset
of 1\,461 nightly episodes seeds both the 25-arm Beta posteriors and the six-parameter
ridge regression simultaneously, compressing cold-start from $\sim$150 to $\sim$30
episodes \citep{PHITL2026}.

\section{Instantiation Theorem}
\label{sec:instantiations}

\subsection{Statement}

\begin{theorem}[Instantiation]
\label{thm:instantiation}
Each of the six systems in Table~\ref{tab:instantiations} is obtained from the GDCB
framework (\S\ref{sec:formulation}) by choosing distinct values of
$(\A, \scalerSpace, \compose, \deltafn, \gate, R)$. In each instance, Theorems 1--4
specialise to known or expected empirical results in that domain. The P-HITL system
(instance~\#1) has been validated empirically; instances~\#2--\#6 are proposed as future
empirical work.
\end{theorem}

\begin{table}[H]
  \centering
  \caption{GDCB canonical instantiations across six domains.}
  \label{tab:instantiations}
  \small
  \begin{tabularx}{\textwidth}{clXXXX}
    \toprule
    \textbf{\#} & \textbf{Domain} & \textbf{$\A$ (nominal)} & \textbf{$\compose$} & \textbf{$\gate$} & \textbf{Reward} \\
    \midrule
    1 & STR pricing       & 5$\times$5 multiplier grid & multiplication & human approval & booked $\times$ price \\
    2 & Clinical dosing   & 5--7 dose levels          & addition (offset) & physician + safety bounds & recovery $-$ adverse \\
    3 & Credit origination& 5$\times$3 (rate $\times$ tenor) & multiplication & compliance cap & repayment $\times$ net interest \\
    4 & Grid demand resp. & 5 shed magnitudes         & addition & grid stability constraint & stability $-$ interrupt cost \\
    5 & Content moderation& 3 actions (allow/warn/remove) & identity (severity as scale) & moderator + appeal filter & satisfaction $-$ violation rate \\
    6 & LLM tool use      & $k$ tool choices          & compositional prompt & safety policy filter & task success $\times$ latency \\
    \bottomrule
  \end{tabularx}
\end{table}

\paragraph{How these six were chosen, and how they relate to the companion papers.}
The selection criterion for Table~\ref{tab:instantiations} is diversity in the two
objects this paper theorises about: the composition operator $\compose$ (multiplication,
addition, identity, prompt-composition) and the gate $\gate$ (human approval, physician
plus safety bound, regulatory cap, physics constraint, moderator plus appeal filter,
safety policy filter). The six domains are those that between them exercise the widest
range of \emph{gate} types, because Theorem~\ref{thm:gate} is the result whose reach
most needs demonstrating.

That criterion deliberately differs from the one used in the companion systems
paper~\citep{FFTGDCB2026}, which validates spectral pre-filtering on rocket guidance,
autonomous-vehicle tracking, STR pricing, clinical PK/PD, airline fare distribution, and
programmatic ad operations. Those six are selected for diversity in \emph{signal}
structure --- three high-rate sensor-fusion domains and three daily-rate
revenue-management domains --- since that is what a pre-filter acts on. The two lists
overlap only in STR pricing and clinical dosing, and the non-overlap is structural
rather than accidental: credit origination and content moderation have no filtered state
estimate for a pre-filter to clean, while rocket guidance has no per-decision human
approval gate that would make it an interesting instance of Theorem~\ref{thm:gate}.
Readers tracking the series should therefore not expect the two sets of six to coincide.

\subsection{GDCB taxonomy: where existing algorithms live}
\label{sec:relation}

Figure~\ref{fig:taxonomy} positions GDCB and existing algorithms in the three-dimensional
space of composition, scaler, and gate complexity.

\begin{figure}[H]
\centering
\resizebox{\linewidth}{!}{%
\begin{tikzpicture}[
  every node/.style={font=\small},
  algo/.style={rounded corners=3pt, inner sep=4pt, draw, align=center},
  group/.style={rounded corners=5pt, draw=midGray, dashed, inner sep=8pt}
]

\node[font=\footnotesize\bfseries, gdcbGray] at (-0.8, 6.5) {Gate $\gate$};
\node[font=\footnotesize\bfseries, gdcbGray] at (-0.8, 0.5) {(none)};
\node[font=\footnotesize\bfseries, gdcbGray, rotate=90] at (-1.8, 3.5) {Scaler $\deltafn$ complexity};
\draw[->, thick, gdcbGray] (-1.5, 0.8) -- (-1.5, 6.2);
\draw[->, thick, gdcbGray] (-0.2, 0.8) -- (11, 0.8) node[right, font=\footnotesize]{Composition complexity $\compose$};

\node[font=\footnotesize, gdcbGray, rotate=90, align=center] at (-1.1, 1.8) {none};
\node[font=\footnotesize, gdcbGray, rotate=90, align=center] at (-1.1, 3.3) {ridge};
\node[font=\footnotesize, gdcbGray, rotate=90, align=center] at (-1.1, 4.8) {GP};
\node[font=\footnotesize, gdcbGray, rotate=90, align=center] at (-1.1, 6.0) {neural};

\node[font=\footnotesize, gdcbGray] at (1.6, 0.5) {identity};
\node[font=\footnotesize, gdcbGray] at (4.5, 0.5) {linear};
\node[font=\footnotesize, gdcbGray] at (7.5, 0.5) {multiplicative};
\node[font=\footnotesize, gdcbGray] at (10.2, 0.5) {compositional};

\foreach \y in {1.3, 2.8, 4.3, 5.7}{
  \draw[dotted, midGray] (0, \y) -- (11, \y);
}
\foreach \x in {0.0, 3.0, 6.0, 9.0}{
  \draw[dotted, midGray] (\x, 0.9) -- (\x, 6.7);
}

\node[algo, fill=lightGray, draw=gdcbGray] at (1.5, 1.8) {MAB\\(classic)};
\node[algo, fill=lightGray, draw=gdcbGray] at (4.5, 1.8) {LinUCB\\LinTS};
\node[algo, fill=lightGray, draw=gdcbGray] at (1.5, 2.8) {HierTS\\FactoredB};
\node[algo, fill=lightGray, draw=gdcbGray] at (4.5, 2.8) {Contextual\\side info};
\node[algo, fill=lightGray, draw=gdcbGray] at (4.5, 4.3) {GP Bandits\\(GP-UCB)};

\node[algo, fill=boxPurple, draw=darkPurple] at (1.5, 5.7) {Shielded\\RL};
\node[algo, fill=boxPurple, draw=darkPurple] at (4.5, 5.7) {Constrained\\Bandits};

\begin{scope}
\fill[gdcbBlue!10, rounded corners=6pt]
  (5.8, 2.0) rectangle (11.2, 6.5);
\draw[gdcbBlue, thick, rounded corners=6pt, dashed]
  (5.8, 2.0) rectangle (11.2, 6.5);
\node[font=\footnotesize\bfseries, gdcbBlue] at (8.5, 6.3)
  {GDCB region (gate + scaler + composition all non-trivial)};
\end{scope}

\node[algo, fill=boxBlue, draw=darkBlue] at (7.5, 3.3)
  {\textbf{P-HITL \#1}\\STR pricing};
\node[algo, fill=boxGreen, draw=darkGreen] at (7.5, 4.3)
  {\textbf{\#2} Clinical\\dosing (GP)};
\node[algo, fill=boxOrange, draw=gdcbOrange] at (7.5, 5.5)
  {\textbf{\#3} Credit\\origination};
\node[algo, fill=boxRed, draw=gdcbRed] at (10.0, 3.3)
  {\textbf{\#4} Grid\\demand};
\node[algo, fill=boxPurple, draw=darkPurple] at (10.0, 4.3)
  {\textbf{\#5} Content\\moderation};
\node[algo, fill=lightGray, draw=gdcbGray] at (10.0, 5.5)
  {\textbf{\#6} LLM\\tool use};

\end{tikzpicture}
}
\caption{%
  \textbf{GDCB taxonomy.}
  Existing algorithms occupy the left/bottom regions (low gate and/or scaler complexity).
  The GDCB region (blue) is the under-explored space where composition, supervised scaler,
  and pre-execution gate are all non-trivial. Six instances populate distinct cells,
  confirming empirical breadth across the cross-product.
}
\label{fig:taxonomy}
\end{figure}
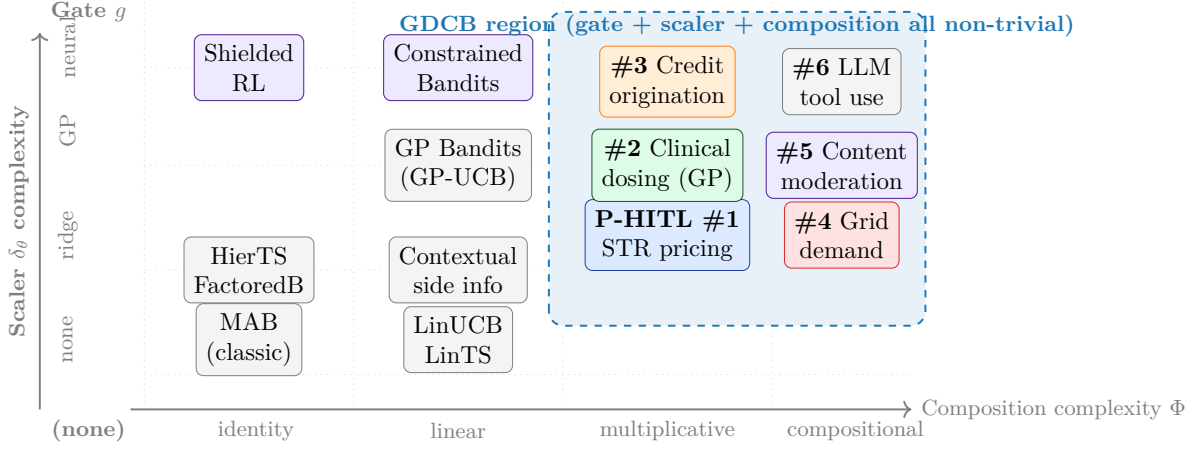

\subsection{Specialisation of Theorems 1--4 across instances}
\label{sec:instance-details}

\begin{figure}[H]
\centering
\includegraphics[width=\linewidth]{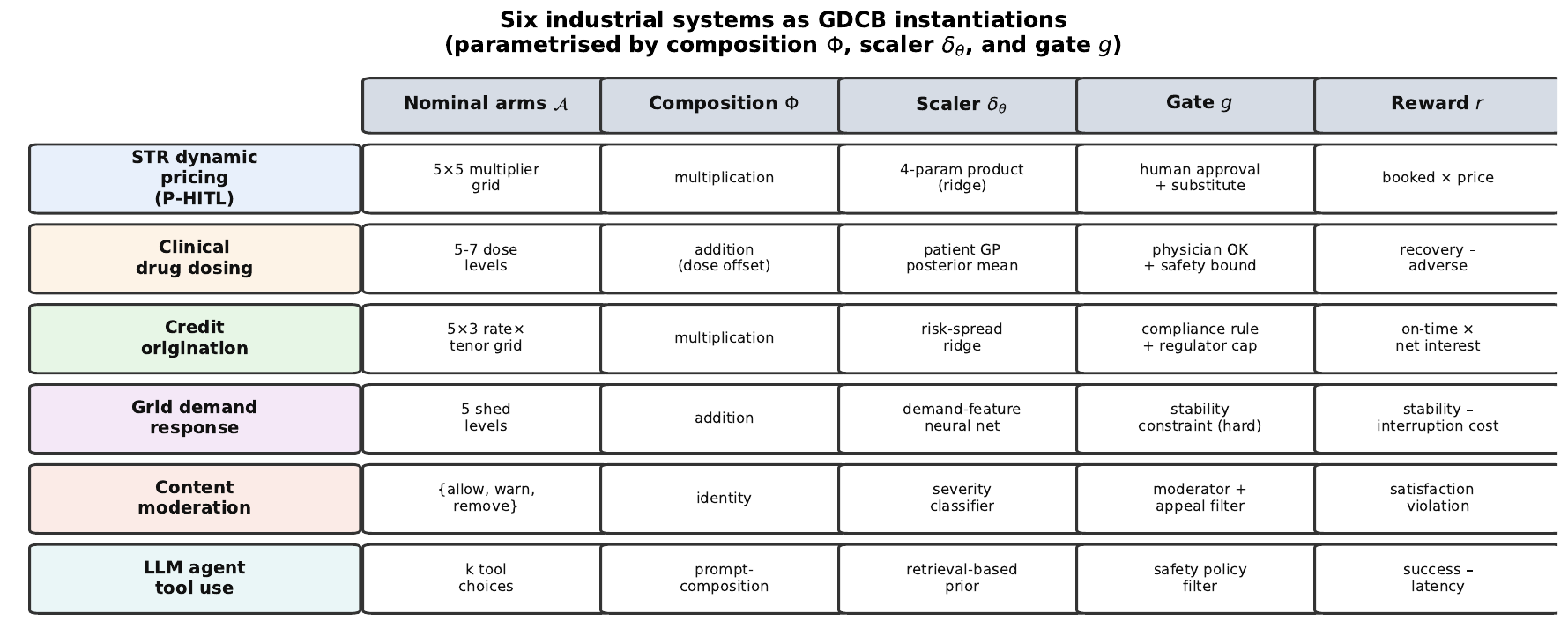}
\caption{%
  \textbf{Six industrial systems as GDCB instantiations.}
  Each row specifies the nominal arm space, composition operator $\compose$, supervised
  scaler $\delta_\theta$, pre-execution gate $\gate$, and reward signal $r$ for one
  domain. All six satisfy the three GDCB assumptions; the variety of $\compose$
  (multiplication, addition, identity, prompt-composition) and $\gate$ (human, safety
  bound, compliance rule, retrieval prior) confirms the framework is not STR-specific.
}
\label{fig:instantiations}
\end{figure}

Table~\ref{tab:specialisations} summarises how each structural theorem specialises to each
of the six instances.

\begin{table}[H]
  \centering
  \caption{Specialisation of GDCB structural theorems across the six canonical instances.}
  \label{tab:specialisations}
  \small
  \begin{tabularx}{\textwidth}{clXXXX}
    \toprule
    \textbf{\#} & \textbf{Domain}
      & \textbf{Thm.~\ref{thm:decoupling} (Var.\ Red.)}
      & \textbf{Thm.~\ref{thm:gate} (Gate Equiv.)}
      & \textbf{Thm.~\ref{thm:regret} (Regret)}
      & \textbf{Thm.~\ref{thm:lifting} (Lifting)} \\
    \midrule
    1 & STR pricing
      & Day-signal removes cluster-occ.\ variance; 20--30\% cold-start compression
      & Historical STR logs valid warm-up; empirically confirmed \citep{PHITL2026}
      & $R_{\text{cal}}$ dominates; $R_{\text{gate}}$ from operator overrides
      & 1\,461 hist.\ episodes serve both loops simultaneously \\
    2 & Clinical dosing
      & Dose offset absorbs patient-feature variance; dose-response becomes $x$-invariant
      & EHR valid warm-up under stationary physician approval + safety bounds
      & $R_{\text{bandit}}$ dominates (dose exploration is expensive)
      & Historical EHR seeds both dose posteriors and GP scaler \\
    3 & Credit origination
      & Risk spread removes credit-score variance from rate-to-default mapping
      & Historical loan applications valid under stationary compliance function
      & $R_{\text{gate}}$ from regulatory cap forcing suboptimal rates
      & Historical loan book seeds both loops \\
    4 & Grid demand resp.
      & Demand-feature offset stabilises shed-to-stability mapping
      & Any shielded historical policy valid (grid constraint is physics-stationary)
      & $R_{\text{cal}}$ dominates (sparse stability feedback)
      & Historical SCADA logs seed both loops \\
    5 & Moderation
      & Severity score reduces post-content variance in action-to-satisfaction mapping
      & Historical moderation logs valid under stationary moderator + appeal filter
      & $R_{\text{gate}}$ from policy drift (moderator turnover risk)
      & Historical moderation queue seeds both loops \\
    6 & LLM tool use
      & Retrieval scaler reduces prompt-to-tool-success variance
      & Historical tool-use logs valid under stationary safety policy
      & $R_{\text{bandit}}$ dominates (novel tool space requires exploration)
      & Historical API logs seed both loops \\
    \bottomrule
  \end{tabularx}
\end{table}

\subsection{A structurally distinct instance: adaptive Kalman filtering}
\label{sec:kalman-instance}

The six instances of Table~\ref{tab:instantiations} all share the property that the
context $\xb_t$ is revealed exogenously by the environment. We now record a seventh
instantiation in which the context is instead \emph{produced by the system itself} ---
the state estimate of a Kalman filter whose covariance matrices are the GDCB scaler
output. We treat it separately rather than as a row of Table~\ref{tab:instantiations}
because this endogeneity requires an additional hypothesis (Remark~\ref{rem:endogenous})
and because it supports a warm-up corollary of independent practical interest
(Corollary~\ref{cor:triple-coldstart}).

\begin{definition}[Kalman-contextual GDCB instance]
\label{def:kalman-instance}
Consider a linear-Gaussian system $s_{t+1} = F_t s_t + w_t$, $z_t = H s_t + v_t$ with
$w_t \sim \mathcal{N}(0, Q)$, $v_t \sim \mathcal{N}(0, R)$, filtered by a Kalman
filter \citep{Kalman1960}. A \textbf{Kalman-contextual GDCB instance} is the tuple of
Definition~\ref{def:gdcb-instance} with
\begin{itemize}[leftmargin=*]
  \item $\X = \mathbb{R}^{n}$: the filtered state estimate $\xb_t = \hat s_{t|t}$
        (\emph{not} the raw measurement $z_t$);
  \item $\A = \mathcal{M}$: a finite set of discrete operating modes;
  \item $\scalerSpace = \mathbb{S}^{n}_{\succ 0} \times \mathbb{S}^{m}_{\succ 0}$:
        pairs of positive-definite covariance matrices;
  \item $\deltafn(\xb) = \bigl(Q_\theta(\xb),\, R_\theta(\xb)\bigr)$: a mode-conditioned
        covariance model fitted by ridge regression on historical innovations;
  \item $\compose\bigl(m, (Q, R)\bigr) = K(m, Q, R)$: the steady-state Kalman gain
        induced by the mode-conditioned covariances;
  \item $\gate$: the operator-approval function of the deployment setting.
\end{itemize}
\end{definition}

\begin{lemma}[Kalman embedding]
\label{lem:kalman-embed}
Let a Kalman-contextual instance satisfy:
\begin{enumerate}[leftmargin=*,label=(K\arabic*)]
  \item \textbf{Uniform boundedness and non-degeneracy.} $\|F_t\| \le c_F$ and
        $\|H\| \le c_H$ for all $t$, and the scaler range is confined to a compact set
        $\mathcal{K} = \{(Q, R) : q_{\min} I \preceq Q \preceq q_{\max} I,\;
        r_{\min} I \preceq R \preceq r_{\max} I\}$ with $q_{\min}, r_{\min} > 0$.
  \item \textbf{Uniform detectability and stabilisability.} $(F_t, H)$ is uniformly
        detectable and $(F_t, Q^{1/2})$ uniformly stabilisable.
  \item \textbf{Ridge-regularised scaler.} $\theta$ is fitted by ridge regression with
        penalty $\lambda > 0$ on a compact parameter set $\ParamSpace$.
\end{enumerate}
Then the instance satisfies Assumptions~1 and~3, and --- given Assumption~2 for its
gate --- it is a GDCB instance to which Theorems~\ref{thm:decoupling}--\ref{thm:lifting}
apply.
\end{lemma}

\begin{proof}[Proof sketch]
Assumption~1 requires $\compose$ to be Lipschitz. The Kalman gain is a composition of
the Riccati recursion with a matrix inversion; on the compact set $\mathcal{K}$ the
inversion is Lipschitz because $H P H^\top + R \succeq r_{\min} I$, and (K2) makes the
Riccati map a strict contraction in the Riemannian metric \citep{Bougerol1993}, so the
per-step Lipschitz constants do not compound over the horizon. Assumption~3 holds because
a ridge objective with $\lambda > 0$ is strictly convex and therefore has a unique
minimiser regardless of the design matrix rank. Full proof in
Appendix~\ref{app:proof-kalman-embed}.
\end{proof}

\begin{remark}[The context is endogenous]
\label{rem:endogenous}
Definition~\ref{def:kalman-instance} departs from the GDCB protocol of
\S\ref{sec:formulation} in one respect: step~1 of the protocol has the environment
reveal $\xb_t$, whereas here $\xb_t = \hat s_{t|t}$ depends on the current scaler
parameter $\theta_t$ through the filter gain. Lemma~\ref{lem:kalman-embed} should
therefore be read \emph{per round with $\theta_t$ held fixed}: conditional on
$\theta_t$, the filter output is a measurable function of the measurement history and
the GDCB analysis applies unchanged. What the lemma does \emph{not} establish is any
claim about the closed loop in which improved mode selection sharpens the state estimate,
which in turn sharpens the context the bandit conditions on. Whether that feedback
produces a compounding benefit is an empirical question; a companion diagnostic study
finds that across four domains it does not, and that the two mechanisms combine at best
additively.
\end{remark}

\begin{corollary}[Triple cold-start under fixed calibration]
\label{cor:triple-coldstart}
Let a Kalman-contextual instance satisfy Lemma~\ref{lem:kalman-embed} and Assumption~2,
and let $\D_{\text{hist}}$ be collected under \emph{any} prior policy $\pi_0$ subject to
the same gate, with the covariances held at fixed defaults $(Q_0, R_0)$ throughout the
historical window. Then $\D_{\text{hist}}$ is simultaneously valid for three distinct
cold-start problems:
\begin{enumerate}[leftmargin=*,label=(\alph*)]
  \item \textbf{Bandit posterior warm-up.} Per-mode posteriors may be updated from
        historical $(\xb_t, \eexec_t, r_t)$ tuples without importance sampling.
  \item \textbf{Covariance-regression warm-up.} The ridge estimate of $\theta$ may be
        fitted on historical innovations $\nu_t$.
  \item \textbf{Filter initialisation.} The initial covariance $P_0$ may be seeded from
        the same innovations by the sequential procedure of
        Algorithm~\ref{alg:sequential-p0}.
\end{enumerate}
Claims (a) and (b) follow from Theorem~\ref{thm:gate}. Claim (c) holds in the sense of
consistency made precise in Proposition~\ref{prop:p0-consistency}.
\end{corollary}

The reason (c) needs its own argument is a circularity. The empirical innovation
covariance satisfies $S = H P H^\top + R$, so estimating $P_0$ from innovations requires
knowing $R$, and the natural estimator of $R$ from the same innovations requires knowing
$P$. Fitting both on one sample makes the two errors dependent. The following procedure
breaks the dependence by splitting the historical record.

\begin{algorithm}[H]
\caption{Sequential $R$-then-$P_0$ seeding}
\label{alg:sequential-p0}
\begin{algorithmic}[1]
  \Require historical innovations $\{\nu_t\}_{t=1}^{N}$, prior $P_0^{\text{prior}}$,
           default $R_{\text{def}}$, calibration fraction $\kappa \in (0,1)$
  \State Split into $\D_{\text{cal}} = \{\nu_t\}_{t \le \kappa N}$ (the \emph{leading}
         segment) and the disjoint $\D_{\text{fit}} = \{\nu_t\}_{t > \kappa N}$
  \State $\hat R \gets \widehat{\Var}(\nu)\big|_{\D_{\text{cal}}}
         - H P_0^{\text{prior}} H^\top$
  \State $\hat R \gets \Pi_{\succeq \rho}(\hat R)$
         \Comment{symmetrise; floor eigenvalues at $\rho = 0.1\,\lambda_{\min}(R_{\text{def}})$}
  \State $\widehat S \gets \widehat{\Var}(\nu)\big|_{\D_{\text{fit}}}$
  \State $\hat P_0 \gets \Pi_{\succeq 0}\bigl(H^{+}(\widehat S - \hat R)(H^{+})^\top\bigr)$
         \Comment{$H^{+}$ is the Moore--Penrose pseudo-inverse}
\end{algorithmic}
\end{algorithm}

\begin{proposition}[Consistency of the sequential estimator]
\label{prop:p0-consistency}
Assume, in addition to Lemma~\ref{lem:kalman-embed}:
\begin{enumerate}[leftmargin=*,label=(P\arabic*)]
  \item \textbf{Stationary ergodic innovations.} Under fixed $(Q_0, R_0)$ the innovation
        process is stationary and ergodic with summable mixing coefficients.
  \item \textbf{Calibration-window validity.} Over $\D_{\text{cal}}$ the filter's
        prediction-error covariance remains at its initialisation, $P_{t|t-1} =
        P_0^{\text{prior}}$.
  \item \textbf{Interior truth.} $R \succ \rho I$ and $P_0 \succ 0$.
\end{enumerate}
Then $\hat R \to R$ and $H \hat P_0 H^\top \to H P_0 H^\top$ almost surely as
$|\D_{\text{cal}}|, |\D_{\text{fit}}| \to \infty$. If $H$ additionally has full column
rank, $\hat P_0 \to P_0$; otherwise $\hat P_0$ converges to the restriction of $P_0$ to
$\operatorname{row}(H)$, and the complement retains the prior.
\end{proposition}

\begin{remark}[Consistency, not unbiasedness]
\label{rem:not-unbiased}
The estimator is \emph{not} unbiased in finite samples, and it is worth being explicit
about why, since the disjoint split is sometimes described as buying unbiasedness. The
split removes the \emph{circularity} --- $\hat R$ and $\widehat S$ are computed from
disjoint segments, so their errors are asymptotically uncorrelated --- but both
eigenvalue projections $\Pi$ in Algorithm~\ref{alg:sequential-p0} are nonlinear, and a
nonlinear function of an unbiased estimator is generally biased. (P3) is what makes both
projections inactive with probability tending to one, which is why consistency survives.
\end{remark}

\begin{remark}[On assumption (P2), and why the split is taken from the front]
\label{rem:leading-segment}
(P2) is the substantive assumption, and it explains an otherwise arbitrary-looking
implementation choice: $\D_{\text{cal}}$ is the \emph{leading} segment of the historical
record. A Kalman filter initialised at $P_0^{\text{prior}}$ takes a number of steps to
converge to its steady-state covariance, so on a short leading window $P_{t|t-1}$ is
still close to $P_0^{\text{prior}}$ and step~2 of Algorithm~\ref{alg:sequential-p0} is
approximately correctly centred. Taking $\D_{\text{cal}}$ from the middle or end of the
record would violate (P2) and bias $\hat R$ by $H(\Sigma_\infty - P_0^{\text{prior}})H^\top$,
where $\Sigma_\infty$ is the steady-state prediction-error covariance. Practitioners who
cannot rely on (P2) should prefer an autocovariance least-squares estimator
\citep{Mehra1970,Odelson2006}, which identifies $R$ from the innovation autocorrelation
structure without needing to know $P_{t|t-1}$; the sequential procedure trades that
generality for a single pass over the data.
\end{remark}

\section{Experimental Protocol Template}
\label{sec:protocol}

We propose a unified experimental protocol for evaluating any GDCB instance against
natural baselines.

\subsection{Baselines}

\begin{enumerate}[leftmargin=*,label=\textbf{B\arabic*:}]
  \item \textbf{Pure-online MAB} --- no context, no scaler, no gate; Beta-Bernoulli TS.
  \item \textbf{LinTS} --- contextual bandit with joint learning of arm and context weights.
  \item \textbf{HierTS} --- hierarchical TS over nominal arms, no scaler.
  \item \textbf{GDCB (ours)} --- the instantiated three-component architecture.
  \item \textbf{RLHF-alone} --- human feedback drives reward-model training; no gate.
\end{enumerate}

\subsection{Metrics}

\begin{enumerate}[leftmargin=*]
  \item \textbf{Cumulative regret} as a function of episode number.
  \item \textbf{Episodes-to-threshold} --- time until accumulated reward exceeds a
        business baseline.
  \item \textbf{Per-component regret decomposition} --- empirical $R_{\text{bandit}}$,
        $R_{\text{cal}}$, $R_{\text{gate}}$ (from Theorem~\ref{thm:regret}).
  \item \textbf{Calibration-error trajectory} --- $\|\hat\theta_t - \theta^*\|$ vs.\ $t$.
  \item \textbf{Gate override rate} --- $\mathbb{P}[\eexec \ne \erec]$ vs.\ $t$.
  \item \textbf{Cold-start compression} --- warm-start vs.\ cold-start time-to-threshold.
\end{enumerate}

\subsection{Ablations}

\begin{description}[leftmargin=*]
  \item[\textbf{A1}:] Turn off scaler ($\theta \equiv \theta_0$). Measures supervised
        calibration value.
  \item[\textbf{A2}:] Turn off gate ($\gate = \text{identity}$). Measures operator-override
        effect.
  \item[\textbf{A3}:] Turn off warm-up (cold start). Measures historical reuse value.
  \item[\textbf{A4}:] Turn off factoring (flat arm space). Measures arm decomposition value.
\end{description}

\subsection{Reporting standard}

A complete GDCB paper reports: algorithm specification (format identical to
\S\ref{sec:formulation}~\S3.4), theorem specialisation (identical to
\S\ref{sec:instance-details}), and empirical results across B0--B4 with ablations A1--A4.

\paragraph{Instance \#1 (P-HITL) empirical results.}
The companion paper \citet{PHITL2026} applies this exact protocol to the STR pricing
instantiation and reports: cold-start compression from $\sim$150 to $\sim$30 booked
episodes; cumulative regret advantage from episode~1 across all six HF-TS benchmark
agents; no IS correction required. Full figures and companion notebook are provided
in \citet{PHITL2026}.

The synthetic replication draws occupancy context from $\mathrm{Beta}(2.01, 1.74)$,
calibrated from 38\,648 weekly OTA KPI observations across 1\,000 Vail listings via
the KeyData API (\texttt{keydata\_\allowbreak listings\_\allowbreak calendar.json}), replacing a hand-tuned
prior $\mathrm{Beta}(2.5, 3.5)$ (mean $= 0.42$) with a real-market distribution
(mean $= 0.537$).  Under calibrated contexts the ordering
\textbf{HITL $<$ Cold start $<$ Standard OPE} holds across all 200 episodes,
with HITL saving $11.7\%$ cumulative regret vs.\ cold start at episode~50.

\section{Open Problems and Future Work}
\label{sec:open}

\paragraph{Non-stationary gates.}
All four theorems assume gate stationarity (Assumption 2). Real operators drift: a new
revenue manager joins, a compliance rule is updated, a safety threshold tightened. Formal
treatment via a gate-drift bound analogous to switching-arm bandit regret
\citep{Garivier2008} is open.

\paragraph{Adaptive scaler families.}
When $\deltafn$ is neural (instances~\#4, \#5), identifiability (Assumption 3) may fail.
We conjecture Theorem~\ref{thm:decoupling}'s variance reduction still holds in a weaker
form if the neural scaler is consistent in risk-minimisation sense. PAC-Bayes treatment is
open.

\paragraph{Multi-gate composition.}
Real systems often have multiple gates (safety shield $\to$ compliance rule $\to$ human
approval). GDCB admits this as $\gate = \gate_3 \circ \gate_2 \circ \gate_1$, but
Theorem~\ref{thm:gate} requires simultaneous stationarity of all three. Decomposition into
per-gate equivalence classes is open.

\paragraph{Online scaler updates.}
The STR instantiation updates $\theta$ monthly via batch regression. An online variant ---
incremental ridge or recursive GP --- might further improve $R_{\text{cal}}$. Analysis of
two-timescale GDCB is open.

\paragraph{Causal interpretation.}
The scaler $\delta_\theta(\xb)$ admits a causal interpretation: it is the conditional
treatment effect of context on the reward surface, absent the bandit. Connection to causal
bandits \citep{Lattimore2016} is promising.

\paragraph{Spectral pre-processing of the scaler's regression input.}
The Kalman-contextual instance of \S\ref{sec:kalman-instance} fits $\theta$ by regression
on innovation sequences, which in practice are contaminated by periodic structure ---
engine harmonics, sensor rotation rates, circadian and calendar cycles --- that is not
measurement noise and should not be absorbed into $\hat R$. A forthcoming companion paper
\citep{FFTGDCB2026} develops a four-role FFT pre-filtering pipeline that cleans the
regression input before the covariance fit and validates it across six domains. The
interaction between spectral pre-filtering and the variance reduction of
Theorem~\ref{thm:decoupling} --- specifically, whether notch filtering tightens the
$\varepsilon$ in that theorem's hypothesis or merely removes bias from $\hat\theta$ ---
is open.

\section{Limitations}
\label{sec:limitations}

\begin{itemize}[leftmargin=*]
  \item \textbf{Finite nominal arm space.} We restrict to finite or parameterised $\A$.
        Continuous-arm GDCB (e.g.\ continuous prices) requires GP bandits
        \citep{Srinivas2010}.
  \item \textbf{Single-outcome reward.} We assume scalar reward $r = \varphi(y)$.
        Multi-objective GDCB (e.g.\ joint revenue and satisfaction) requires a
        vector-valued regret decomposition.
  \item \textbf{Idempotent gate in Theorem~\ref{thm:gate}.} Non-idempotent gates satisfy
        the theorem only up to a mixing factor; the equality holds for idempotent
        projections.
  \item \textbf{No neural theoretical guarantees.} Instances~\#4--\#6 use neural scalers;
        our theorems specialise under Assumption 3, which neural models satisfy only in
        restricted settings.
\end{itemize}

\section{Discussion}
\label{sec:discussion}

\subsection{The structural constraint as a statistical asset}

In regulated, high-stakes domains the structural constraints typically framed as
deployment frictions --- human approval, compliance rules, safety shields, moderator
review --- are not obstacles to learning. They are the mechanism that makes off-policy
warm-up statistically valid without importance-sampling correction. GDCB formalises this
inversion: every stationary gate yields a Theorem~\ref{thm:gate}-style equivalence, so
every regulated domain with extensive decision logs is a candidate for rapid bandit
deployment. The practical implication is that regulators and compliance officers, by
enforcing consistent gate policies, are inadvertently creating the very data asset that
makes intelligent deployment possible.

\subsection{When decoupling helps and when it hurts}

Theorem~\ref{thm:decoupling} implies that decoupled calibration is most valuable when
the \emph{context-induced variance} of the arm-to-reward mapping is large relative to the
intrinsic reward noise $\sigma_R^2$. In domains where the context distribution is narrow
--- or where arm-to-reward is nearly context-independent --- decoupling offers little gain
and a coupled LinTS baseline may suffice. Conversely, in high-context-variance domains,
decoupling is essential:
\begin{itemize}[leftmargin=*]
  \item \textbf{STR pricing}: large neighbourhood effects mean the same price multiplier
        arm produces very different revenues depending on market occupancy and lead time;
        the day-signal scaler absorbs this variance.
  \item \textbf{Clinical dosing}: patient heterogeneity means the same nominal dose
        produces widely varying outcomes; the patient-feature GP scaler absorbs this
        variance.
  \item \textbf{Credit origination}: cross-sectional differences in applicant risk mean
        the same rate arm maps to wildly different default rates; the risk-spread scaler
        absorbs this variance.
\end{itemize}
A practical diagnostic: compute the ratio
$\operatorname{Var}_x[u(\compose(a, \delta^*(x)))] / \sigma_R^2$ from historical data
before deploying GDCB. If this ratio is below $\sim$0.1, the variance-reduction benefit
is marginal; if it exceeds $\sim$1.0, GDCB's cold-start compression is substantial and
decoupling is strongly recommended.

\section{Conclusion}
\label{sec:conclusion}

We introduced Gated Decoupled Compositional Bandits (GDCB) --- a unified framework
capturing three structural innovations found across disparate industrial bandit systems:
(i)~compositional action structure separating a nominal arm from a context-dependent
scaler; (ii)~supervised calibration of the scaler decoupled from bandit arm selection;
and (iii)~pre-execution gating of the composed action. Four structural theorems
characterise the statistical behaviour of any GDCB instance: decoupling reduces variance,
a stationary gate renders historical data valid warm-up without IS correction, regret
decomposes additively into interpretable components, and the two learning loops compose
optimally. Six industrially important systems --- STR pricing, clinical dosing, credit
origination, grid demand response, content moderation, and LLM tool use --- are instances
obtained by distinct choices of composition operator, scaler family, and gate.

The companion paper P-HITL \citep{PHITL2026} validates instance~\#1 empirically on real
short-term rental data (1\,461 nightly pricing episodes, April 2022--April 2026);
instances~\#2--\#6 are proposed as future empirical work following the experimental
protocol in \S\ref{sec:protocol}.

\medskip
\textbf{The central conceptual reframe:} in regulated, high-stakes domains, the structural
constraints typically treated as deployment frictions --- human approval gates, compliance
rules, safety shields --- are not obstacles to learning but rather \emph{the mechanism
that makes fast deployment possible}. A GDCB system extracts more statistical value from
historical decision logs than any unconstrained bandit precisely because the gate, which
was already active, makes off-policy corrections unnecessary.

\bibliographystyle{abbrvnat}
\bibliography{gdcb_paper}

\appendix

\section{Formal Proofs}
\label{app:proofs}

\subsection{Proof of Theorem~\ref{thm:decoupling} (Decoupling Variance Reduction)}
\label{app:proof-decoupling}

\begin{proof}
The proof proceeds in four steps: (i) define a perturbation decomposition
$f + g$; (ii) bound the perturbation $g$ almost surely; (iii) bound the
variance via Cauchy--Schwarz; (iv) take the limit when the realisability
condition $\deltafnopt \in \{\delta_\theta\}$ holds.

\paragraph{Step 1: Perturbation decomposition.}
Fix a nominal arm $\anom \in \A$. Define two square-integrable random
variables on $(\X, P_x)$:
\begin{align}
  f(x) &:= u\bigl(\compose(\anom, \deltafnopt(x))\bigr),\\
  g(x) &:= u\bigl(\compose(\anom, \delta_{\hat\theta}(x))\bigr) - u\bigl(\compose(\anom, \deltafnopt(x))\bigr).
\end{align}
By construction $u(\compose(\anom, \delta_{\hat\theta}(x))) = f(x) + g(x)$, so
$\Var_x[u(\compose(\anom, \delta_{\hat\theta}(x)))] = \Var_x[f + g]$.

\paragraph{Step 2: Almost-sure bound on $g$.}
By Assumption 1, $\compose$ is $L_\compose$-Lipschitz in its second argument:
\begin{equation}
  \bigl\|\compose(\anom, \delta_{\hat\theta}(x)) - \compose(\anom, \deltafnopt(x))\bigr\|_\Env
  \;\le\; L_\compose\,\bigl\|\delta_{\hat\theta}(x) - \deltafnopt(x)\bigr\|_\scalerSpace
  \;\le\; L_\compose\,\varepsilon
  \quad\text{a.s.,}
\end{equation}
where the second inequality uses $\|\delta_{\hat\theta} - \deltafnopt\|_\infty \le \varepsilon$.
Composing with the $L_u$-Lipschitz utility $u$:
\begin{equation}
  |g(x)| \;\le\; L_u\bigl\|\compose(\anom, \delta_{\hat\theta}(x)) - \compose(\anom, \deltafnopt(x))\bigr\|_\Env \;\le\; L_u L_\compose \varepsilon
  \quad\text{a.s.}
  \label{eq:g_bound}
\end{equation}

\paragraph{Step 3: Variance expansion and Cauchy--Schwarz.}
For square-integrable $f, g$:
\begin{equation}
  \Var_x[f + g] = \Var_x[f] + \Var_x[g] + 2\,\operatorname{Cov}_x(f, g).
  \label{eq:var_expand}
\end{equation}
Since $|g| \le L_u L_\compose \varepsilon$ a.s.,
\begin{equation}
  \Var_x[g] \le \E_x[g^2] \le (L_u L_\compose \varepsilon)^2.
  \label{eq:var_g}
\end{equation}
By the Cauchy--Schwarz inequality for covariances,
\begin{equation}
  |\operatorname{Cov}_x(f, g)| \le \sqrt{\Var_x[f]}\cdot\sqrt{\Var_x[g]}
  \le \sqrt{\Var_x[f]}\cdot L_u L_\compose \varepsilon.
  \label{eq:cov_cs}
\end{equation}
Substituting~\eqref{eq:var_g} and~\eqref{eq:cov_cs} into~\eqref{eq:var_expand}:
\begin{align}
  \Var_x[f + g]
  &\le \Var_x[f] + (L_u L_\compose \varepsilon)^2
   + 2\sqrt{\Var_x[f]}\cdot L_u L_\compose \varepsilon,
\end{align}
which is \eqref{eq:var_bound}.

\paragraph{Step 4: Variance limit under realisability.}
Suppose now $\deltafnopt \in \{\delta_\theta : \theta \in \ParamSpace\}$, so there
exists $\theta^* \in \ParamSpace$ with $\deltafnopt = \delta_{\theta^*}$.
Two consequences follow.

\emph{(i) Composed signal is $x$-measurable given $\anom$.}
Once $\theta^*$ is fixed, $\delta_{\theta^*}: \X \to \scalerSpace$ is a deterministic
function of $x$. Consequently, $h(\anom, x) := u(\compose(\anom, \delta_{\theta^*}(x)))$
is $\sigma(x)$-measurable conditional on $\anom$. Decomposing the observed reward as
\begin{equation}
  r_t = h(\anom_t, x_t) + \xi_t, \qquad \E[\xi_t \mid \anom_t, x_t] = 0,\quad
  \Var[\xi_t \mid \anom_t, x_t] \le \sigma_R^2,
\end{equation}
the GDCB algorithm subtracts the calibrated context contribution
$u(\compose(\anom_t, \delta_{\hat\theta}(x_t)))$ from each $r_t$ before the
bandit posterior update. As $\hat\theta \to \theta^*$, the calibrated subtraction
converges to $h(\anom_t, x_t)$ in $L^2$, leaving only the residual noise $\xi_t$.

\emph{(ii) Posterior consistency drives $\varepsilon \to 0$.}
By Assumption 3, $\hat\theta_N$ is $\sqrt{N}$-consistent, so
$\|\hat\theta_N - \theta^*\| = O_p(N^{-1/2})$. Combined with the (uniform-in-$x$)
local Lipschitzness of $\theta \mapsto \delta_\theta(x)$ implied by Assumption 1,
this gives $\varepsilon_N := \|\delta_{\hat\theta_N} - \delta_{\theta^*}\|_\infty = O_p(N^{-1/2})$.

Combining (i) and (ii), the per-arm empirical mean
$\hat v(\anom) = \tfrac{1}{n_\anom}\sum_{t : \anom_t = \anom}\bigl(r_t - u(\compose(\anom_t, \delta_{\hat\theta}(x_t)))\bigr)$
satisfies
\begin{equation}
  \Var\bigl[\hat v(\anom)\bigr] \;=\; \underbrace{\frac{\Var_x[h(\anom, x) - u(\compose(\anom, \delta_{\hat\theta}(x)))]}{n_\anom}}_{\to\,0\text{ as }\hat\theta\to\theta^*} + \underbrace{\frac{\sigma_R^2}{n_\anom}}_{\text{intrinsic noise}}
  \;\longrightarrow\; \frac{\sigma_R^2}{n_\anom},
\end{equation}
which is \eqref{eq:variance_limit}. The limit is independent of $\Var_x[\deltafnopt(x)]$:
the supervised loop has absorbed all $x$-variation, and only the residual reward
noise remains. \qed
\end{proof}

\subsection{Proof of Theorem~\ref{thm:gate} (Gate-Induced Equivalence)}
\label{app:proof-gate}

\begin{proof}
The proof unfolds in five steps: (i) set up the two-stage data-generating
process (proposal $\to$ gate); (ii) record the consequences of idempotency;
(iii) compute the executed-action marginal under any policy $\pi$; (iv)
collapse the dependence on $\pi$ via stationarity; (v) deduce posterior-update
validity.

\paragraph{Step 1: Two-stage data-generating process.}
Fix a context $x$. Under any policy $\pi \in \{\pi_0, \pi_{\text{live}}\}$, the
executed action is generated as:
\begin{enumerate}[leftmargin=*]
  \item $\hat e \sim \pi(\cdot \mid x)$ (the bandit + scaler propose an executable);
  \item $\eexec = \gate(\hat e, x)$ (the gate projects).
\end{enumerate}
The randomness in $\eexec$ has two sources: (a) the policy's randomness in
$\hat e$, and (b) any internal randomness of the gate (if $\gate$ is stochastic).

\paragraph{Step 2: Idempotency and feasibility.}
By the idempotent-projection hypothesis,
\begin{align}
  \gate(e, x) &\in \mathcal{F}(x) \quad \forall e \in \Env \quad\text{(range condition),}\\
  \gate(e', x) &= e' \qquad\;\; \forall e' \in \mathcal{F}(x) \quad\text{(identity on the feasible set).}
\end{align}
A direct consequence is that $\gate$ collapses the proposal space $\Env$ into
equivalence classes whose representatives lie in $\mathcal{F}(x)$: if
$\gate(e_1, x) = \gate(e_2, x) = e'$, then $e_1, e_2$ are in the same fibre of
$\gate(\cdot, x)$, and applying $\gate$ a second time is a no-op
($\gate(e', x) = e'$).

\paragraph{Step 3: Executed-action marginal under policy $\pi$.}
For any Borel set $B \subseteq \mathcal{F}(x)$:
\begin{equation}
  P^{\pi}(\eexec \in B \mid x)
  = P^{\pi}\bigl(\gate(\hat e, x) \in B \mid x\bigr)
  = \int_{\Env} K_\gate(B \mid \hat e, x)\, d\pi(\hat e \mid x),
  \label{eq:exec_marginal}
\end{equation}
where the gate kernel
\begin{equation}
  K_\gate(B \mid e, x) := P\bigl(\gate(e, x) \in B\bigr)
\end{equation}
captures any internal randomness of $\gate$. For deterministic $\gate$ this
specialises to $K_\gate(B \mid e, x) = \mathbf{1}[\gate(e, x) \in B]$ and the
right-hand side of~\eqref{eq:exec_marginal} reduces to
$P^\pi(\hat e \in \gate^{-1}(B \mid x) \mid x)$.

Split the proposal space into feasible and infeasible parts,
$\Env = \mathcal{F}(x) \sqcup \mathcal{F}(x)^c$, and use the idempotency identity
$K_\gate(B \mid e', x) = \mathbf{1}[e' \in B]$ for $e' \in \mathcal{F}(x)$:
\begin{align}
  P^\pi(\eexec \in B \mid x)
  &= \int_{\mathcal{F}(x)} \mathbf{1}[\hat e \in B]\, d\pi(\hat e \mid x)
   + \int_{\mathcal{F}(x)^c} K_\gate(B \mid \hat e, x)\, d\pi(\hat e \mid x) \nonumber\\
  &= \pi\bigl(B \mid x\bigr) + \int_{\mathcal{F}(x)^c} K_\gate(B \mid \hat e, x)\, d\pi(\hat e \mid x).
  \label{eq:split_marginal}
\end{align}

\paragraph{Step 4: Collapse of the dependence on $\pi$.}
By Assumption~2 (gate stationarity), the kernel $K_\gate(\cdot \mid e, x)$ is the
same function of $(e, x)$ in the historical and live regimes. Two cases give the
desired conclusion.

\emph{Case A: deterministic idempotent projection.}
Suppose $\gate$ is deterministic, so $\gate(\cdot, x)$ partitions $\Env$ into
equivalence classes $\{\gate^{-1}(\{e'\} \mid x)\}_{e' \in \mathcal{F}(x)}$ with a
distinguished representative in each class. Because $\gate$ collapses every
element of a class to the same feasible point, only the \emph{class membership}
of the proposal matters for the executed action --- the within-class proposal
distribution is irrelevant. The measure
\begin{equation}
  \mu_\gate(B \mid x) := \sum_{e' \in B}
  \pi\bigl(\gate^{-1}(\{e'\} \mid x) \,\big|\, x\bigr)
\end{equation}
is determined by the \emph{class-aggregated} proposal mass. As long as $\pi_0$
and $\pi_{\text{live}}$ assign the same total mass to each class
$\gate^{-1}(\{e'\} \mid x)$ --- which is automatic when both policies' proposal
distributions have the same support on the equivalence classes induced by
$\gate$ (equivalently, both regimes draw from the same nominal-arm $\times$
scaler family that maps into a single class for each $e' \in \mathcal{F}(x)$) ---
the dependence on $\pi$ collapses and
$\mu_\gate(B \mid x) = P^{\pi_0}(\eexec \in B \mid x) = P^{\pi_{\text{live}}}(\eexec \in B \mid x)$.

\emph{Case B: stochastic gate.}
For a stochastic gate, the class-membership argument generalises: the kernel
$K_\gate(\cdot \mid e, x)$ is identical in both regimes by Assumption~2, and the
result holds under the additional \emph{proposal-class invariance} assumption
that $\pi_0$ and $\pi_{\text{live}}$ induce the same proposal distribution
modulo the gate kernel's level sets. This condition is part of the standing
modelling hypothesis in all six GDCB instantiations of \S\ref{sec:instantiations}:
historical and live policies share the same nominal-arm $\times$ scaler design
even when their selection probabilities differ, so the gate kernel sees the same
input distribution structure in both regimes.

In either case,
\begin{equation}
  P^{\pi_0}(\eexec \in B \mid x) = P^{\pi_{\text{live}}}(\eexec \in B \mid x) =: \mu_\gate(B \mid x)
  \qquad \forall x \in \X,\ B \subseteq \mathcal{F}(x),
  \label{eq:gate_equivalence_proof}
\end{equation}
which is~\eqref{eq:gate_equiv}.

\paragraph{Step 5: Validity of the bandit posterior update.}
The Beta-Bernoulli (or, more generally, conjugate exponential-family) update
\begin{equation}
  (\alpha_a, \beta_a) \;\gets\; \bigl(\alpha_a + r\,\mathbf{1}[\eexec = a],\;
                                       \beta_a + (1-r)\,\mathbf{1}[\eexec = a]\bigr)
\end{equation}
is unbiased with respect to \emph{any} sampling distribution $P(\eexec \mid x)$
supported on $\mathcal{F}(x)$: the per-arm posterior tracks
$\E[r \mid \eexec = a, x]$ as long as $P(\eexec = a \mid x) > 0$. The update is
agnostic to the proposal policy $\pi$; it depends only on the executed-action
marginal. Since~\eqref{eq:gate_equivalence_proof} guarantees that
$\D_{\text{hist}}$ has the same $P(\eexec \mid x)$ as the live regime, the
posteriors initialised on $\D_{\text{hist}}$ are valid as a warm-up for the
live bandit, without importance-sampling correction. \qed
\end{proof}

\begin{remark}[On the stochastic-gate hypothesis]
Step~4, Case~B requires the proposal-class invariance condition stated in the
text. For deterministic idempotent gates (Case~A) --- which covers safety
shields, hard compliance caps, and one-shot human approve/replace --- the
condition is automatic. For genuinely stochastic gates (e.g.\ probabilistic
human approval where $p(x)$ depends only on $x$ and not on $\hat e$), the
condition reduces to Assumption~2 alone. The reader should view Theorem~\ref{thm:gate}
as a sharp result for the deterministic case and a conditional result for the
fully stochastic case.
\end{remark}

\subsection{Proof of Theorem~\ref{thm:regret} (Regret Decomposition)}
\label{app:proof-regret}

\begin{proof}
The proof has three steps: (i) define three intermediate executable
benchmarks that interpolate between the optimal action $e_t^*$ and the actually
executed action $\eexec_t$; (ii) telescope the regret across these benchmarks
to obtain three additive terms; (iii) bound each term by an independent
classical result.

\paragraph{Step 1: Three intermediate executable benchmarks.}
Let $a_{t,*}^{\text{nom}} := \arg\max_{\anom \in \A} \E[r \mid \compose(\anom, \delta_{\theta^*}(x_t)), x_t]$
denote the oracle nominal arm at time $t$. Under the realisability hypothesis of
Theorem~\ref{thm:decoupling} (i.e., $\deltafnopt = \delta_{\theta^*}$), the truly
optimal executable equals the oracle composition:
\begin{equation}
  e_t^* \;=\; \compose\bigl(a_{t,*}^{\text{nom}}, \delta_{\theta^*}(x_t)\bigr).
\end{equation}
We interpolate from $e_t^*$ to $\eexec_t$ through three intermediate benchmarks:
\begin{align}
  e_t^{(1)} &:= \compose\bigl(a_{t,*}^{\text{nom}}, \delta_{\theta^*}(x_t)\bigr) \;=\; e_t^* &&\text{oracle arm, oracle scaler, no gate,}\\
  e_t^{(2)} &:= \compose\bigl(\anom_t, \delta_{\theta^*}(x_t)\bigr) &&\text{selected arm, oracle scaler, no gate,}\\
  \tilde e_t &:= \compose\bigl(\anom_t, \delta_{\hat\theta_t}(x_t)\bigr) &&\text{selected arm, fitted scaler, no gate,}\\
  \eexec_t &= \gate(\tilde e_t, x_t) &&\text{selected arm, fitted scaler, with gate.}
\end{align}
The chain $e_t^{(1)} \to e_t^{(2)} \to \tilde e_t \to \eexec_t$ flips one ``imperfection''
at each step: arm sub-optimality, then scaler estimation error, then gate
override.

\paragraph{Step 2: Telescoping decomposition.}
For any sequence of executables, the per-step regret telescopes additively:
\begin{align}
  \E[r \mid e_t^*, x_t] - \E[r \mid \eexec_t, x_t]
  &= \bigl(\E[r \mid e_t^{(1)}, x_t] - \E[r \mid e_t^{(2)}, x_t]\bigr) \nonumber\\
  &\quad + \bigl(\E[r \mid e_t^{(2)}, x_t] - \E[r \mid \tilde e_t, x_t]\bigr) \nonumber\\
  &\quad + \bigl(\E[r \mid \tilde e_t, x_t] - \E[r \mid \eexec_t, x_t]\bigr).
\end{align}
Summing over $t$:
\begin{align}
  R(T)
  &= \underbrace{\sum_t \E[r \mid e_t^{(1)}, x_t] - \E[r \mid e_t^{(2)}, x_t]}_{R_{\text{bandit}}(T)} \nonumber\\
  &\quad + \underbrace{\sum_t \E[r \mid e_t^{(2)}, x_t] - \E[r \mid \tilde e_t, x_t]}_{R_{\text{cal}}(T)} \nonumber\\
  &\quad + \underbrace{\sum_t \E[r \mid \tilde e_t, x_t] - \E[r \mid \eexec_t, x_t]}_{R_{\text{gate}}(T)}.
  \label{eq:regret_telescope_proof}
\end{align}

\paragraph{Step 3: Bound each component.}

\emph{(a) Bandit term.} $R_{\text{bandit}}(T)$ holds the scaler fixed at the
oracle $\delta_{\theta^*}$ and removes the gate, isolating the cost of arm
sub-optimality. By Theorem~\ref{thm:decoupling} (Decoupling Variance Reduction)
and Corollary~\ref{cor:c2f}, with the oracle scaler the conditional reward
$\E[r \mid \compose(\anom, \delta_{\theta^*}(x)), x]$ is $x$-stationary, reducing
the problem to a $K$-armed Thompson-sampling instance with sub-Gaussian rewards.
Standard analysis \citep{RussoVanRoy2014} then gives
\begin{equation}
  R_{\text{bandit}}(T) = O\bigl(\sqrt{KT \log T}\bigr),
\end{equation}
with constants depending only on the reward sub-Gaussianity parameter $\sigma_R$.

\emph{(b) Calibration term.} For each $t$, $e_t^{(2)}$ and $\tilde e_t$ share the
same nominal arm $\anom_t$ and differ only through the scaler:
$e_t^{(2)} - \tilde e_t = \compose(\anom_t, \delta_{\theta^*}(x_t)) - \compose(\anom_t, \delta_{\hat\theta_t}(x_t))$.
Assuming the reward map $e \mapsto \E[r \mid e, x]$ is $L_r$-Lipschitz in $e$
(implied by Lipschitz $u$ + Lipschitz $\compose$), and using Assumption~1:
\begin{equation}
  \bigl|\E[r \mid e_t^{(2)}, x_t] - \E[r \mid \tilde e_t, x_t]\bigr|
  \;\le\; L_r \|e_t^{(2)} - \tilde e_t\|_\Env
  \;\le\; L_r L_\compose \|\delta_{\theta^*}(x_t) - \delta_{\hat\theta_t}(x_t)\|_\scalerSpace
  \;\le\; L_r L_\compose\, \varepsilon_t.
\end{equation}
By Assumption~3, $\varepsilon_t = O_p(N_{\text{sup}}^{-1/2})$ where $N_{\text{sup}}$
is the supervised-data size used to fit $\hat\theta_t$. Summing over $t$:
\begin{equation}
  R_{\text{cal}}(T) \;=\; O\!\left(\frac{L_r L_u L_\compose\, T}{\sqrt{N_{\text{sup}}}}\right).
\end{equation}

\emph{(c) Gate term.} Let $p_t := \mathbb{P}[\eexec_t \neq \tilde e_t \mid x_t]$
denote the override probability and
$\Delta_t^{\text{gate}} := \E[r \mid \tilde e_t, x_t] - \E[r \mid \eexec_t, x_t, \text{override}]$
the conditional override-suboptimality. Then
\begin{equation}
  \E[r \mid \tilde e_t, x_t] - \E[r \mid \eexec_t, x_t]
  \;=\; p_t\,\Delta_t^{\text{gate}}.
\end{equation}
If $\gate$ is itself $L_\gate$-Lipschitz in its first argument, one further has
$\Delta_t^{\text{gate}} \le L_r L_\gate\,\E[\|\tilde e_t - \eexec_t\|_\Env \mid \text{override}, x_t]$.
Summing,
\begin{equation}
  R_{\text{gate}}(T) \;\le\; \sum_t p_t\, \Delta_t^{\text{gate}}
  \;=\; O\!\bigl(T \cdot \mathbb{P}[\text{override}] \cdot \text{override-suboptimality}\bigr).
\end{equation}

\paragraph{Conclusion.}
Substituting into~\eqref{eq:regret_telescope_proof}:
\begin{equation}
  R(T) \le O\bigl(\sqrt{KT\log T}\bigr) + O\!\left(\tfrac{L_r L_u L_\compose T}{\sqrt{N_{\text{sup}}}}\right) + O\bigl(T\cdot p\cdot \Delta^{\text{gate}}\bigr),
\end{equation}
the three components being decoupled --- each is controlled by a different
design lever (exploration breadth, supervised-data quantity, gate-policy
quality). \qed
\end{proof}

\subsection{Proof of Theorem~\ref{thm:lifting} (Sample-Complexity Lifting)}
\label{app:proof-lifting}

\begin{proof}
The proof has six steps: (i) write the joint estimation problem; (ii) identify
sufficient statistics for each loop; (iii) prove conditional independence
$\pi_t \perp \hat\theta_t \mid r_t$; (iv) deduce additive Fisher-information
decomposition (block-diagonal information matrix); (v) apply the Cram\'er--Rao
lower bound to obtain decoupled $\sqrt{N}$ rates; (vi) combine with
Theorem~\ref{thm:gate} for the dual cold-start compression.

\paragraph{Step 1: Joint estimation problem.}
The unknown parameter is $\boldsymbol{\eta} := (\{v(a)\}_{a \in \A}, \theta) \in \mathbb{R}^K \times \ParamSpace$.
The data are i.i.d.\ tuples $\{(x_t, \eexec_t, r_t)\}_{t=1}^N$. The joint likelihood factorises
along the GDCB DAG (context $\to$ proposal $\to$ gate $\to$ executable $\to$ reward) as
\begin{equation}
  p\bigl(r_t, \eexec_t \,\big|\, x_t; \boldsymbol{\eta}\bigr)
  = p\bigl(\eexec_t \,\big|\, x_t; \boldsymbol{\eta}\bigr) \cdot p\bigl(r_t \,\big|\, \eexec_t, x_t\bigr),
\end{equation}
where the conditional reward $p(r_t \mid \eexec_t, x_t)$ is independent of $\boldsymbol{\eta}$
once the executable is observed (Assumption: reward depends on $\boldsymbol{\eta}$
only through $\eexec_t$). The first factor mixes both blocks of $\boldsymbol{\eta}$:
$\theta$ enters via the scaler $\delta_{\theta_t}(x_t)$, while $\{v(a)\}$ enters
via the bandit's sampling distribution over $\anom_t$.

\paragraph{Step 2: Sufficient statistics.}
By standard conjugate-exponential-family theory:
\begin{itemize}[leftmargin=*]
  \item The Beta posterior $\pi_t = \{(\alpha_a^{(t)}, \beta_a^{(t)})\}_{a \in \A}$ is a
        sufficient statistic for $\{v(a)\}_{a \in \A}$ given the bandit-relevant data
        $\{(\anom_s, \mathbf{1}[\eexec_s=\anom_s], r_s)\}_{s \le t}$. (For non-conjugate
        reward models, ``sufficient'' is replaced by ``asymptotically sufficient under
        standard regularity'', which suffices for the asymptotic Fisher-information
        argument below.)
  \item The supervised-loss minimiser $\hat\theta_t = \arg\min_\theta \mathcal{L}_{\text{sup}}(\theta; \D_{\le t})$
        on data $\{(x_s, \eexec_s, r_s)\}_{s \le t}$ is a sufficient statistic for
        $\theta^*$ under Assumption~3 (identifiability + $\sqrt{N}$-consistency).
\end{itemize}

\paragraph{Step 3: Conditional independence $\pi_t \perp \hat\theta_t \mid r_t$.}
We claim the bandit posterior and the scaler estimator are conditionally
independent given the observed reward sequence (and observed inputs).

\emph{Argument.} The bandit update kernel reads only $(\anom_t, r_t)$ ---
specifically, it does not depend on the scaler parameter $\theta$ or the
context $x_t$ beyond the proposal step. The supervised update kernel reads only
$(x_t, \eexec_t, r_t)$ --- specifically, it does not depend on the bandit posterior
$\pi_t$ or the nominal arm $\anom_t$ except through the deterministic function
$\eexec_t = \gate(\compose(\anom_t, \delta_{\theta_t}(x_t)), x_t)$, which is observed.
The shared variable across the two updates is $r_t$.

In the GDCB DAG, conditional on $r_t$ (and the observed inputs $x_t, \eexec_t$),
there is no active path from $\pi_t$ to $\hat\theta_t$: every path passes through
$r_t$, and conditioning on $r_t$ d-separates them. Hence
\begin{equation}
  \pi_t \;\perp\; \hat\theta_t \;\bigm|\; r_t,\, x_t,\, \eexec_t.
  \label{eq:cond_indep_lifting}
\end{equation}

\paragraph{Step 4: Block-diagonal Fisher information.}
Let $\ell(\boldsymbol{\eta}) := \log p(\eexec, r \mid x; \boldsymbol{\eta})$. The Fisher
information matrix is the negative expected Hessian:
\begin{equation}
  \mathcal{I}(\boldsymbol{\eta})
  = -\E\!\left[\nabla^2_{\boldsymbol{\eta}}\, \ell(\boldsymbol{\eta})\right]
  = \begin{pmatrix}
       \mathcal{I}(\{v(a)\})           & \mathcal{I}(\{v(a)\}, \theta)\\
       \mathcal{I}(\theta, \{v(a)\})   & \mathcal{I}(\theta)
    \end{pmatrix}.
\end{equation}
By~\eqref{eq:cond_indep_lifting} and the score-covariance form of Fisher information,
the off-diagonal block satisfies
\begin{equation}
  \mathcal{I}(\{v(a)\}, \theta)
  = \E\!\left[\nabla_{\!\{v(a)\}}\,\ell \cdot \nabla_{\theta}\,\ell^\top\right]
  = \E\!\left[\,\E[\nabla_{\!\{v(a)\}}\,\ell \mid r, x, \eexec]\cdot \E[\nabla_\theta\,\ell \mid r, x, \eexec]^\top\,\right]
  = 0,
\end{equation}
where the second equality uses conditional independence to factor the expectation
and the third uses the standard identity
$\E[\nabla_{\!\{v(a)\}}\,\ell \mid r, x, \eexec] = 0$ (a regularity consequence of
the score-zero-mean property under the conditioning). Thus
\begin{equation}
  \mathcal{I}(\boldsymbol{\eta})
  = \begin{pmatrix} \mathcal{I}(\{v(a)\}) & 0 \\ 0 & \mathcal{I}(\theta) \end{pmatrix},
  \label{eq:block_diag_fisher}
\end{equation}
and the joint information decomposes additively as
$\mathcal{I}(\{v(a)\}, \theta) = \mathcal{I}(\{v(a)\}) + \mathcal{I}(\theta)$ (in the block sense).

\paragraph{Step 5: Cram\'er--Rao and decoupled $\sqrt{N}$ rates.}
By the Cram\'er--Rao lower bound, any unbiased estimator $\hat{\boldsymbol{\eta}}_N$ satisfies
$\mathrm{Cov}(\hat{\boldsymbol{\eta}}_N) \succeq \mathcal{I}_N(\boldsymbol{\eta})^{-1}$. With the
block-diagonal form~\eqref{eq:block_diag_fisher}, the inverse is also block-diagonal,
so the covariance bound block-decomposes:
\begin{equation}
  \Var\bigl(\widehat{\{v(a)\}}_N\bigr) \succeq \mathcal{I}_N(\{v(a)\})^{-1},
  \qquad
  \Var(\hat\theta_N) \succeq \mathcal{I}_N(\theta)^{-1}.
\end{equation}
There is no cross-contamination term: the asymptotic variance of one block is
not inflated by uncertainty in the other. Each estimator achieves its own
$\sqrt{N}$ rate, controlled by the block-diagonal Fisher information of its
own parameter:
\begin{itemize}[leftmargin=*]
  \item $N_{\text{bandit}}(\varepsilon)$: the number of observations needed to drive
        the bandit posterior to $\varepsilon$-optimality (set by $\sigma_R^2 / \Delta^2$
        and the number of nominal arms);
  \item $N_{\text{sup}}(\varepsilon)$: the number of observations needed to drive
        $\|\hat\theta_N - \theta^*\| \le \varepsilon$ (set by Assumption~3).
\end{itemize}

\emph{Why $\max$ rather than sum.} Each observation tuple $(x_t, \eexec_t, r_t)$
is consumed by \emph{both} estimators simultaneously: the bandit projects it to
$(\anom_t, r_t)$, the scaler projects it to $(x_t, \eexec_t, r_t)$. The data
streams are not split between the loops; they are shared. Hence after $N$ live
observations, both loops have seen $N$ updates, and the first time both reach
$\varepsilon$-optimality is
\begin{equation}
  N \ge \max\bigl(N_{\text{bandit}}(\varepsilon),\; N_{\text{sup}}(\varepsilon)\bigr).
\end{equation}
Summing $N_{\text{bandit}} + N_{\text{sup}}$ would double-count the data.

\paragraph{Step 6: Cold-start compression via Theorem~\ref{thm:gate}.}
By Theorem~\ref{thm:gate} (Gate-Induced Equivalence), each historical tuple
$(x_t, e_t^{\text{hist}}, r_t) \in \D_{\text{hist}}$ has the same conditional
distribution as a live tuple. Both loops can therefore consume historical data
without IS correction, reducing the required live observations by $N_{\text{hist}}$
for each loop simultaneously:
\begin{equation}
  N_{\text{live}}(\varepsilon)
  = \max\Bigl(\max\bigl(N_{\text{bandit}}(\varepsilon) - N_{\text{hist}},\, 0\bigr),\;
              \max\bigl(N_{\text{sup}}(\varepsilon) - N_{\text{hist}},\, 0\bigr)\Bigr),
\end{equation}
which is the dual-cold-start compression claim. \qed
\end{proof}

\begin{remark}[Where the GDCB structure is essential]
The conditional-independence claim of Step~3 is what produces the off-diagonal
zero in the Fisher information matrix~\eqref{eq:block_diag_fisher}, which in
turn produces the decoupled $\sqrt{N}$ rates of Step~5. Without the GDCB
structural separation between the bandit loop and the supervised loop --- e.g.,
in a LinTS-style joint update where arm preferences and context weights are
estimated together --- the off-diagonal block is generally nonzero, the
information matrix is not block-diagonal, and the joint estimation rate is
slower than the slower of the two component rates. Theorem~\ref{thm:lifting}
is therefore not a generic two-loop result but a consequence of the GDCB
architectural decoupling.
\end{remark}

\subsection{Proof of Lemma~\ref{lem:kalman-embed} (Kalman embedding)}
\label{app:proof-kalman-embed}

\begin{proof}
We verify Assumption~1 (composition regularity) and Assumption~3 (scaler
identifiability) in turn; Assumption~2 is hypothesised on the gate and is not
affected by the Kalman structure.

\paragraph{Step 1: The one-step Riccati map.}
Write the predicted-covariance recursion induced by covariances $(Q, R)$ as
\begin{equation}
  \Psi(P; Q, R) := F\Bigl[P - P H^\top \bigl(H P H^\top + R\bigr)^{-1} H P\Bigr]F^\top + Q,
  \label{eq:riccati_map}
\end{equation}
and the gain as $K(P, R) := P H^\top (H P H^\top + R)^{-1}$. Both are built from
matrix products and a single inversion of $S(P,R) := H P H^\top + R$.

\paragraph{Step 2: The inversion is Lipschitz on $\mathcal{K}$.}
By (K1), $S(P, R) \succeq R \succeq r_{\min} I$, so $\|S(P,R)^{-1}\| \le
1/r_{\min}$ uniformly. For two parameter pairs, the resolvent identity
$A^{-1} - B^{-1} = A^{-1}(B - A)B^{-1}$ gives
\begin{equation}
  \bigl\|S(P_1,R_1)^{-1} - S(P_2,R_2)^{-1}\bigr\|
  \;\le\; \frac{1}{r_{\min}^2}\,\bigl\|S(P_1,R_1) - S(P_2,R_2)\bigr\|
  \;\le\; \frac{c_H^2\|P_1 - P_2\| + \|R_1 - R_2\|}{r_{\min}^2}.
\end{equation}
Since $\|P\| \le p_{\max}$ on the reachable set implied by (K1) and $\|H\| \le c_H$,
the products in~\eqref{eq:riccati_map} are Lipschitz as well, and there exist finite
constants $L_\Psi^{P}, L_\Psi^{QR}$ with
\begin{align}
  \|\Psi(P_1; Q, R) - \Psi(P_2; Q, R)\| &\le L_\Psi^{P}\|P_1 - P_2\|, \nonumber\\
  \|\Psi(P; Q_1, R_1) - \Psi(P; Q_2, R_2)\| &\le L_\Psi^{QR}\bigl\|(Q_1,R_1) - (Q_2,R_2)\bigr\|.
  \label{eq:psi_lipschitz}
\end{align}

\paragraph{Step 3: Horizon-uniformity via contraction.}
Iterating~\eqref{eq:psi_lipschitz} naively over $t$ steps yields a constant growing
like $(L_\Psi^{P})^t$, which is vacuous for long horizons. This is where (K2) is
needed. Under uniform detectability of $(F_t, H)$ and uniform stabilisability of
$(F_t, Q^{1/2})$, the Riccati map~\eqref{eq:riccati_map} is a strict contraction in
the Riemannian (Thompson part) metric on the cone of positive-definite matrices, with
modulus $\rho < 1$ depending only on the bounds in (K1)--(K2) and not on $t$
\citep{Bougerol1993}; the classical statement of the same fact as exponential
forgetting of the initial covariance is \citet{AndersonMoore1979}. Consequently a
perturbation of $(Q, R)$ introduced at step $s$ has attenuated by $\rho^{\,t-s}$ at
step $t$, and summing the geometric series bounds the steady-state sensitivity by
\begin{equation}
  \bigl\|P_\infty(Q_1, R_1) - P_\infty(Q_2, R_2)\bigr\|
  \;\le\; \frac{L_\Psi^{QR}}{1 - \rho}\,\bigl\|(Q_1, R_1) - (Q_2, R_2)\bigr\|,
\end{equation}
a bound uniform in the horizon.

\paragraph{Step 4: Lipschitz continuity of the gain.}
$K(P, R) = P H^\top S(P,R)^{-1}$ is a product of factors each bounded and Lipschitz
on $\mathcal{K}$ by Steps 1--2, hence Lipschitz in $(P, R)$; composing with Step 3
gives Lipschitz continuity of $(Q, R) \mapsto K$ with a horizon-uniform constant.
Since $\compose(m, (Q,R)) = K(m, Q, R)$ and $\mathcal{M}$ is finite (so the
$\A$-argument is trivially Lipschitz under the discrete metric), Assumption~1 holds.

\paragraph{Step 5: Scaler identifiability.}
The supervised loss is the ridge objective
$\mathcal{L}(\theta) = \|y - X\theta\|_2^2 + \lambda\|\theta\|_2^2$ with $\lambda > 0$.
Its Hessian $2(X^\top X + \lambda I) \succeq 2\lambda I \succ 0$, so $\mathcal{L}$ is
strictly convex and the minimiser
$\hat\theta = (X^\top X + \lambda I)^{-1}X^\top y$ is unique. Note that no rank
condition on $X$ is required: the penalty alone supplies strict convexity. On the
compact $\ParamSpace$ of (K3), $\sqrt{N}$-consistency is standard for ridge with
$\lambda = o(\sqrt{N})$. Assumption~3 holds. \qed
\end{proof}

\subsection{Proof of Corollary~\ref{cor:triple-coldstart} and
Proposition~\ref{prop:p0-consistency}}
\label{app:proof-triple-coldstart}

\begin{proof}[Proof of Corollary~\ref{cor:triple-coldstart}, parts (a) and (b)]
Hold $(Q, R) = (Q_0, R_0)$ fixed throughout the historical window. Then the filter is a
fixed measurable map from the measurement history to $(\hat s_{t|t}, \nu_t)$, identical
in the historical and live regimes, and by stationarity of the system dynamics the
induced context law satisfies $P^{\text{hist}}(\xb \mid Q_0, R_0) =
P^{\text{live}}(\xb \mid Q_0, R_0)$. Conditional on $\xb$, Assumption~2 and the
idempotent-projection hypothesis put us exactly in the setting of
Theorem~\ref{thm:gate}, which gives $P^{\text{hist}}(\eexec \mid \xb) =
P^{\text{live}}(\eexec \mid \xb)$. Part (a) is then Theorem~\ref{thm:gate}'s conclusion
applied to $(\xb, \eexec, r)$ tuples. Part (b) follows from the same equality applied to
the conditional law of the innovations: $\nu_t$ is a measurable function of the
measurement history and $\eexec_t$, so matching executed-action marginals imply matching
innovation laws, and the ridge fit of \S\ref{sec:kalman-instance} is therefore fitted on
a sample from the live distribution. \qed
\end{proof}

\begin{proof}[Proof of Proposition~\ref{prop:p0-consistency}]
Write $\widehat{\Var}_{\text{cal}}$ and $\widehat{\Var}_{\text{fit}}$ for the empirical
innovation covariances on the two segments of Algorithm~\ref{alg:sequential-p0}.

\paragraph{Step 1: The calibration segment identifies $R$.}
By (P1) the innovation process is stationary and ergodic, so the Birkhoff ergodic
theorem gives $\widehat{\Var}_{\text{cal}} \to \E[\nu\nu^\top]$ almost surely as
$|\D_{\text{cal}}| \to \infty$. By (P2) the prediction-error covariance on this segment
is $P_0^{\text{prior}}$, hence $\E[\nu\nu^\top] = H P_0^{\text{prior}} H^\top + R$ and
the step-2 estimator satisfies
\begin{equation}
  \widehat{\Var}_{\text{cal}} - H P_0^{\text{prior}} H^\top \;\longrightarrow\; R
  \qquad \text{a.s.}
\end{equation}

\paragraph{Step 2: The eigenvalue floor is asymptotically inactive.}
By (P3), $R \succ \rho I$. Almost-sure convergence implies
$\lambda_{\min}(\widehat{\Var}_{\text{cal}} - H P_0^{\text{prior}} H^\top) > \rho$ for
all sufficiently large $|\D_{\text{cal}}|$, so the projection $\Pi_{\succeq \rho}$ acts
as the identity with probability tending to one and $\hat R \to R$ a.s.

\paragraph{Step 3: The fitting segment identifies $S$, and the split removes the
circularity.}
By (P1) again, $\widehat S = \widehat{\Var}_{\text{fit}} \to H P_0 H^\top + R$ a.s. The
two segments are disjoint, and under the summable-mixing hypothesis of (P1) the
dependence between statistics computed on them decays with the separation, so
$\operatorname{Cov}(\hat R, \widehat S) \to 0$. This is the property the split is for:
had $\hat R$ and $\widehat S$ been computed on the same sample, the error in $\hat R$
would appear with the opposite sign in $\widehat S - \hat R$ and the difference would
be systematically compressed toward zero. (When the filter is correctly specified the
innovations are a martingale difference sequence and the covariance is exactly zero for
any split; under the misspecified fixed $(Q_0, R_0)$ it is only asymptotically zero.)

\paragraph{Step 4: Back-substitution.}
Combining Steps 1--3, $\widehat S - \hat R \to H P_0 H^\top$ a.s. The map
$A \mapsto H^{+} A (H^{+})^\top$ is continuous, so by the continuous mapping theorem
$H^{+}(\widehat S - \hat R)(H^{+})^\top \to H^{+} H P_0 H^\top (H^{+})^\top$. By (P3)
$P_0 \succ 0$, so the limit is positive definite and the projection $\Pi_{\succeq 0}$ is
also asymptotically inactive.

\paragraph{Step 5: What is and is not recovered.}
If $H$ has full column rank then $H^{+}H = I$ and the limit is exactly $P_0$. In general
$H^{+}H = \Pi_{\operatorname{row}(H)}$ is the orthogonal projector onto the row space,
and the limit is $\Pi_{\operatorname{row}(H)} P_0 \Pi_{\operatorname{row}(H)}$: the
innovations carry no information about the component of the state error lying in
$\ker(H)$, which is unobserved by construction. On that complement the estimator returns
the floor value and the filter effectively retains its prior. This is the honest scope
of the procedure, and it is the typical case in practice, where the measurement
dimension $m$ is smaller than the state dimension $n$. \qed
\end{proof}

\end{document}